\documentclass[twoside]{article}

\usepackage[accepted]{aistats2024}

\usepackage{amsfonts}
\usepackage{amsmath}
\usepackage{amsthm}
\usepackage{algorithmic}
\usepackage{algorithm}
\usepackage{graphicx}
\usepackage{xcolor}
\usepackage{amsmath} 
\usepackage{multirow}
\usepackage{tikz}
\usepackage{soul}
\usepackage{setspace}
\usepackage{booktabs} 
\usepackage{siunitx}  
\usepackage{url}
\usepackage{subcaption}
\usepackage{hyperref}

\usepackage{imakeidx}
\usepackage{etoolbox}
\usepackage{xcolor}

\hypersetup{
    colorlinks=true,    
    linkcolor=blue,     
    urlcolor=blue,      
    citecolor=blue,     
    linkbordercolor=blue, 
    pdfborderstyle={/S/U/W 1}, 
}

\newcommand{\normdist}[2]{\mathcal{N}(#1, #2)}
\newcommand{\boldnum}[1]{\mathbf{#1}}

\newcommand{\jm}[1]{\textcolor{black}{#1}}
\newcommand{\final}[1]{\textcolor{black}{#1}}

\newtheorem{theorem}{Theorem}
\newtheorem{lemma}{Lemma}

\newcounter{oldtocdepth}

\newcommand{\hidefromtoc}{%
  \setcounter{oldtocdepth}{\value{tocdepth}}%
  \addtocontents{toc}{\protect\setcounter{tocdepth}{-10}}%
}

\newcommand{\unhidefromtoc}{%
  \addtocontents{toc}{\protect\setcounter{tocdepth}{\value{oldtocdepth}}}%
}

\begin{document}

%

%

\twocolumn[

\aistatstitle{Training Implicit Generative Models via an Invariant Statistical Loss}

\aistatsauthor{ José Manuel de Frutos \And 
Pablo M. Olmos \And 
Manuel A. Vázquez \And 
Joaquín Míguez}

\aistatsaddress{ Universidad Carlos III\\\texttt{\fcolorbox{red}{white}{jofrutos@ing.uc3m.es}} \And  
Universidad Carlos III \And
Universidad Carlos III \And
Universidad Carlos III
} 
]

\begin{abstract}
Implicit generative models have the capability to learn arbitrary complex data distributions. On the downside, training requires telling apart real data from artificially-generated ones using adversarial discriminators, leading to unstable training and mode-dropping issues. As reported by Zahee et al. (2017), even in the one-dimensional (1D) case, training a \jm{generative adversarial network} (GAN) is challenging and often suboptimal. \final{In this work, we develop a discriminator-free method for training one-dimensional (1D) generative implicit models and subsequently expand this method to accommodate multivariate cases.} Our loss function is a discrepancy measure between a suitably chosen transformation of the model samples and a uniform distribution; hence, it is invariant with respect to the true distribution of the data. We first formulate our method for 1D random variables, providing an effective solution for approximate reparameterization of arbitrary complex distributions. Then, we consider the temporal setting \final{(both univariate and multivariate)}, in which we model the conditional distribution of each sample given the history of the process. We demonstrate through numerical simulations that this new method yields promising results, successfully learning true distributions in a variety of scenarios and mitigating some of the well-known problems that state-of-the-art implicit methods present.
\end{abstract}

\hidefromtoc
\section{Introduction}

\begin{figure*}[h]
    \centerline{
    
    \begin{subfigure}{0.31\textwidth}
        \centering
        \includegraphics[height=4cm, width=\linewidth]{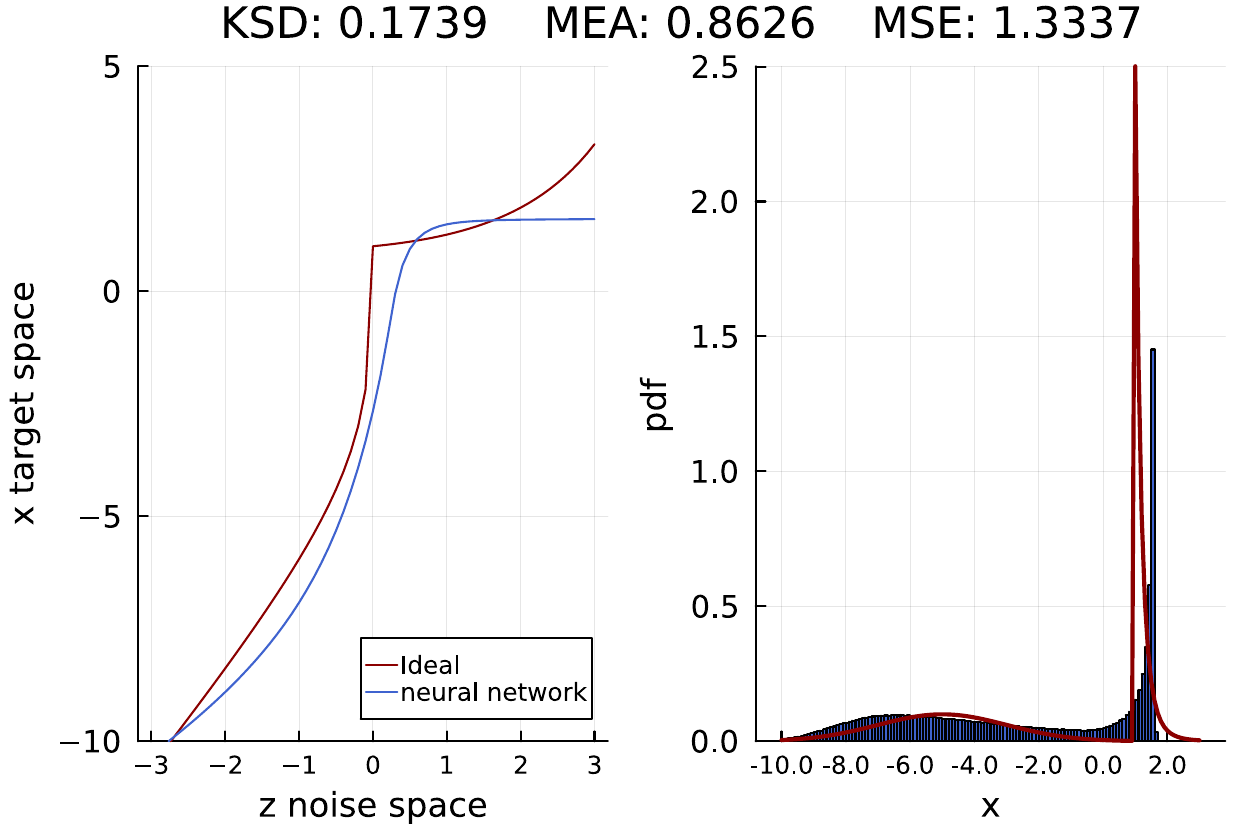}
        \caption{\tiny ISL estimation}
        \label{fancy figure:sub_a}
    \end{subfigure}
    
    \begin{subfigure}{0.31\textwidth}
        \centering
        \includegraphics[height=4cm, width=\linewidth]{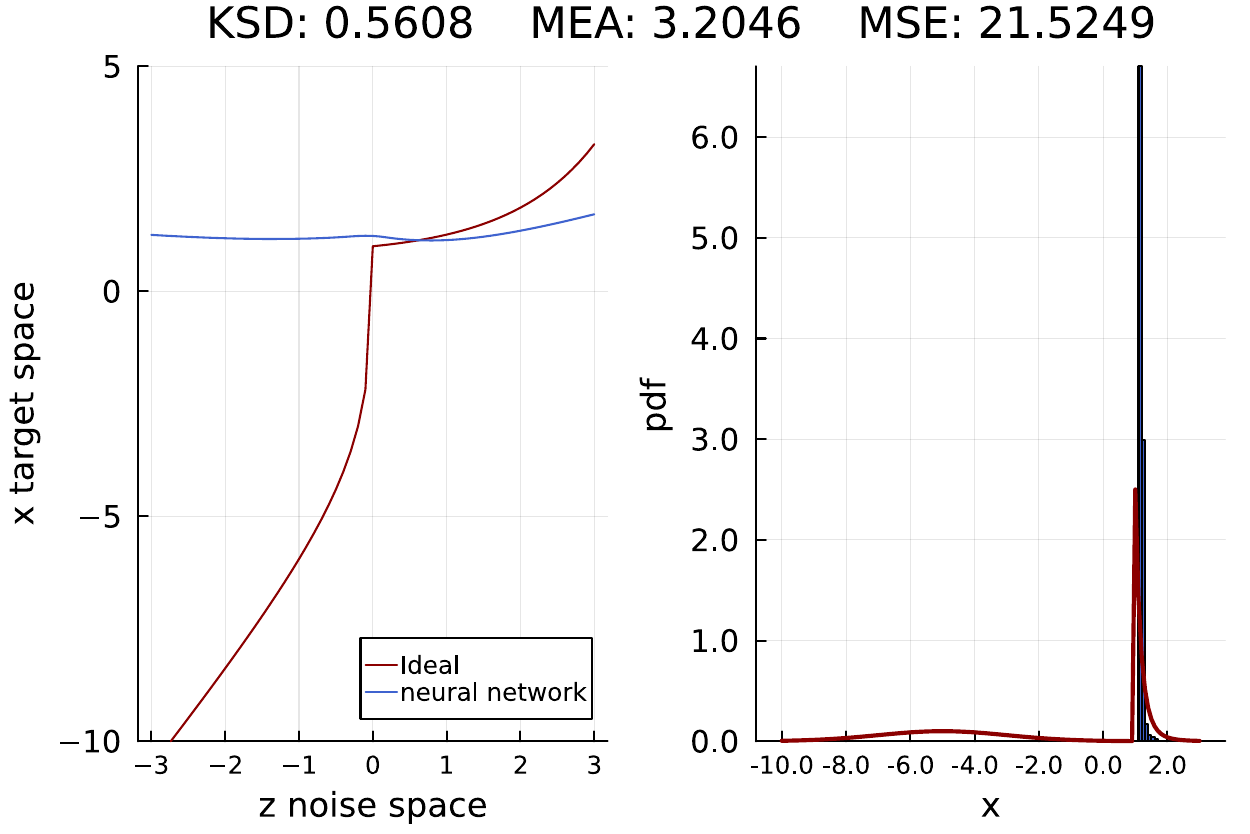}
        \caption{\tiny MMD-GAN estimation}
        \label{fancy figure:sub_b}
    \end{subfigure}
    
    \begin{subfigure}{0.31\textwidth}
        \centering
        \includegraphics[height=4cm, width=\linewidth]{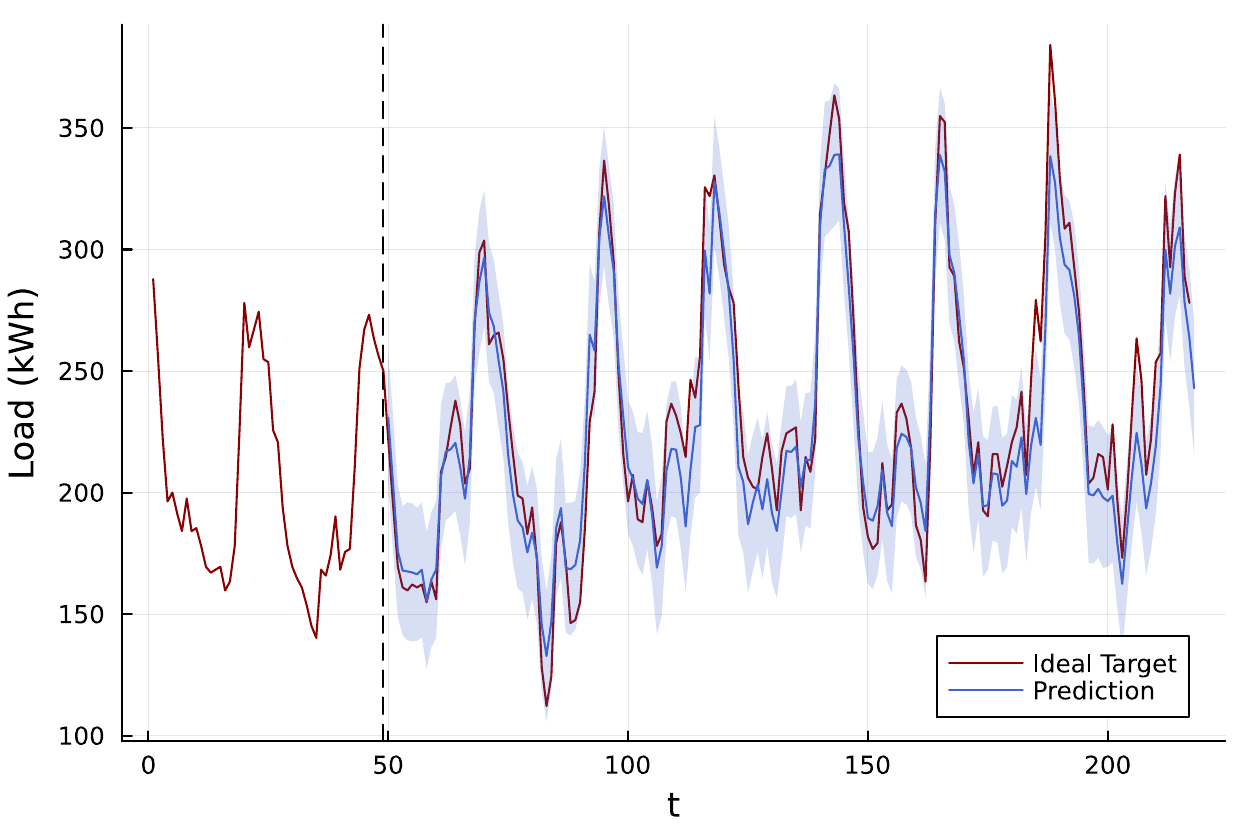}
        \caption{\tiny 7-Day Electricity-C ISL Forecast}
        \label{fancy figure:sub_c}
    \end{subfigure}}
    
    \caption{\small In (a), we show the ISL estimation of the pdf of a 1D random variable with a true distribution that is a mixture model with two components: a Gaussian $\mathcal{N}(-5,2)$ distribution and a Pareto(5,1) \jm{distribution}. In (b), we show the performance of MMD-GAN in \cite{zaheer2017gan} for the same target. In the left-side of both figures, we compare the optimal transformation from a $N(0,1)$ to the target distribution versus the generator learned by ISL (a) and MMD-GAN (b). In (c), we show the temporal \jm{ISL-based} 7-Day Forecast in the Electricity-C dataset \cite{misc_electricityloaddiagrams20112014_321}.}
    \label{fancy figure}
\end{figure*}

Implicit generative models employ \jm{an $m$-dimensional latent random variable (r.v.) $\mathbf{z}$ to simulate random samples from a prescribed \final{$n$-dimensional} target probability distribution. To be precise, the latent variable} undergoes a transformation through a deterministic function $g_{\theta}$, which \final{maps $\mathbb{R}^m \mapsto \mathbb{R}^n$} using the parameter set $\theta$. Given the \jm{model} capability to generate samples with ease, various techniques can be employed for contrasting two sample collections: one originating from the genuine data distribution and the other from the \jm{model} distribution. This approach essentially constitutes a methodology for \jm{the approximation of probability distributions} via comparison.
Generative adversarial networks (GANs) \cite{goodfellow2014generative}, $f$-GANs \cite{nowozin2016f}, Wasserstein-GANs (WGANs) \cite{arjovsky2017wasserstein}, \jm{adversarial variational} Bayes (AVB) \cite{mescheder2017adversarial}, and \jm{maximum mean-miscrepancy} (MMD) GANs \cite{li2017mmd} are some \jm{popular} methods that fall within this framework. 

Approximation of \jm{1-dimensional (1D)} parametric distributions is a seemingly naive problem for which the above-mentioned models can perform below expectations. In \cite{zaheer2017gan}, the authors report that various types of GANs struggle to approximate relatively simple distributions from samples, emerging with MMD-GAN as the most promising technique. However, the latter implements a kernelized extension of a moment-matching criterion defined over a reproducing kernel Hilbert space, and consequently, the objective function is expensive to compute. 

In this work, we introduce a novel approach to train univariate implicit models that relies on a fundamental property of rank statistics. Let \(r_1 < r_2 < \cdots < r_k\) be a ranked (ordered) sequence of independent and identically distributed (i.i.d.) samples from the generative model with probability density function (pdf) \(\tilde{p}\), and let \(y\) be a random sample from a pdf \(p\). If \(\tilde{p} = p\), \jm{then \final{\(\mathbb{P}(r_{i-1} \leq y < r_{i}) = \frac{1}{K}\)} for every \(i = 1, \ldots, K+1\), with the convention that $r_0=-\infty$ \final{and $r_{K+1}=\infty$} (see, e.g., \cite{rosenblatt1952remarks} or \cite{elvira2016adapting} for a short explicit proof)}. This invariant property holds for any continuously distributed data, i.e., for any data with a pdf \(p\). Consequently, even if \(p\) is unknown, we can leverage this invariance to construct an objective (loss) function.  This objective function eliminates the need for a discriminator, directly measuring the discrepancy of the transformed samples \jm{with respect to (w.r.t.)} the uniform distribution. \jm{The computational cost of evaluating this loss increases} linearly with both $K$ and $N$, allowing \jm{for} low-complexity mini-batch updates. Moreover, \jm{the proposed criterion} is invariant across true data distributions, \jm{hence we refer to the resulting objective function as invariant statistical loss (ISL).} \jm{Because of this property, the ISL can be exploited to learn multiple modes} in mixture models and different time steps when learning temporal processes. \final{Additionally, considering the marginal distributions independently, it is straightforward to extend ISL to the multivariate case.} In both scenarios, as illustrated in Figure \ref{fancy figure}, our approach enhances training robustness by avoiding the min-max optimization problem and outperforms GANs, including MMD-GAN and WGAN, and state-of-the-art deep temporal generative models.

\section{Preliminaries}

In this section we formally state the problem of density estimation, and introduce the main ideas that led to the design of the proposed loss function.

\subsection{Problem statement}
\jm{Let $y_1, \ldots, y_N$ be real i.i.d. samples from a probability distribution with pdf $p$, let $z$ be a random sample from a canonical distribution with pdf $p_0$ (e.g., standard Gaussian) and let $g_\theta:\mathbb{R}\mapsto\mathbb{R}$ be a real map parametrized by a $\theta \in \mathbb{R}^l$. The aim is to use the data $y_1, \ldots, y_N$ to learn the parameters $\theta$ such that $\tilde y = g_\theta(z)$ is distributed according to the pdf of the data, $p$.}

\subsection{Implicit generative models}

Implicit generative models train a \textit{generator} neural network (NN), parametrized by \(\theta\) and represented as \(g_{\theta}\), to transform samples \(z \sim \jm{p_0}\) into \(\tilde{y}=g_{\theta}(z)\) with pdf \( \tilde{p}\) such that, ideally, \(p \approx \tilde{p}\). Throughout the training process, the emphasis is placed on computing the discrepancy, \(d\), between $p$ and $\tilde{p}$, which becomes the loss function of an optimization problem. This discrepancy reaches zero if and only if $p=\tilde{p}$. In this paper, we design a loss function \(d\) that measures the discrepancy between \(p\) and \(\tilde{p}\), drawing upon the existence of a statistic that \jm{is uniform whenever} the densities are equal. \jm{In particular,} the proposed statistic is invariant w.r.t. $p$ and it is based on the comparison between a set of \jm{i.i.d samples} generated by our model \jm{and the real data.}

\section{Uniform rank statistics}
\label{Uniform rank statistics}

In Section \ref{Discrete uniform rank statistics} we construct a simple rank statistic that can be proved to be uniform for any i.i.d. random  sample from a continuous distribution with pdf $p$. Then, in \jm{Section} \ref{Approximately uniform statistics} we prove that one can use data from an unknown pdf $p$ to construct rank statistics that are \jm{approximately} uniform when $\tilde{p}\approx p$ (and converge to uniform r.v.'s when $\tilde{p}\to p$, \jm{where $\tilde p$ is the pdf of the generative model}). 

Finally, in Section \ref{A converse theorem} we prove that when the proposed rank statistics are uniformly distributed then $p=\tilde{p}$ almost everywhere. 

\subsection{Discrete uniform rank statistics}\label{Discrete uniform rank statistics}

Let $\tilde{y}_1, \ldots, \tilde{y}_K$ be a random sample from a univariate real distribution with pdf $\tilde{p}$ and let $y$ be a single random sample independently drawn from another distribution with pdf $p$. We construct the set,
\begin{align*}
\mathcal{A}_{K} := \Big\{\tilde{y} \in \{\tilde{y}_{k}\}_{k=1}^{K} : \tilde{y} < y\Big\},
\end{align*}
and the statistic $A_{K} := |\mathcal{A}_{K}|$, i.e., $A_{K}$ is the number of elements of $\mathcal{A}_{K}$. We note that $A_{K}$ is a rank statistic that can be alternatively constructed by first finding indices $j_{1}, \ldots, j_k$ that sort the sample $\tilde{y}_{1}, \ldots, \tilde{y}_{K}$ in ascending order, 
\begin{align*}
    \tilde{y}_{j_{1}}<\tilde{y}_{j_{2}}<\ldots<\tilde{y}_{j_{K}},
\end{align*}
and then \jm{setting}
\begin{align}\label{eq:AK}
     A_{K} = \sup\Big\{ i\in \{1,\ldots,K\}: \tilde{y}_{j_i}<y\Big\}.
\end{align}
Also note that \jm{almost surely $y\not = \tilde{y}_{i}$ for all $i$}. The statistic $A_{K}$ is a discrete r.v. that takes values on the set $\{0, \ldots, K\}$ and we denote its  probability mass function (pmf) as $\mathbb{Q}_{K}: \{0,\ldots,K\} \,\jm{\mapsto}\, [0,1]$. This pmf satisfies the following fundamental result.

\begin{theorem}\label{theorem1}
If $p=\tilde{p}$ then $\mathbb{Q}_{K}(n)=\frac{1}{K+1}$ $\forall n \in \{0,\ldots, K\}$, i.e., $A_{K}$ is a discrete uniform r.v. on  the set $\{0,\ldots,K\}$.
\end{theorem}
\begin{proof}
    See \cite{elvira2021performance} for an explicit proof. This is a basic result that appears under different forms in the literature, e.g., \jm{in} \cite{rosenblatt1952remarks} or \cite{djuric2010assessment}.
\end{proof}

\subsection{Approximately uniform statistics} \label{Approximately uniform statistics}

For $f:\mathbb{R} \jm{\mapsto} \mathbb{R}$ a real function and $p$ the pdf of an univariate real r.v. let us \jm{denote
\begin{align*}
    (f,p) := \int_{-\infty}^{\infty} f(x)p(x) dx
\end{align*}
and let $B_{1}(\mathbb{R}) :=\Big\{(f:\mathbb{R}\to \mathbb{R}) : \sup_{x\in \mathbb{R}}|f(x)|\leq 1\Big\}$ be the set of real functions bounded by $1$.}

\jm{The quantity $\sup_{f\in B_{1}(\mathbb{R})}|(f,p)-(f,\tilde{p})|$ can be directly related to the total variation distance between the distributions with pdf's $p$ and $\tilde p$. Hence, the following theorem states that if the generator output pdf, $\tilde{p}$, is close to $p$ then the rank statistic $A_{K}$ with pmf $\mathbb{Q}_{K}(n)$ constructed in Section \ref{Discrete uniform rank statistics} is approximately uniform.}
\begin{theorem}\label{theorem2}
If $\sup_{f\in B_{1}(\mathbb{R})}|(f,p)-(f,\tilde{p})|\leq \epsilon$ then,
\begin{align*}
    \dfrac{1}{K+1}-\epsilon \leq \mathbb{Q}_{K}(n)\leq \dfrac{1}{K+1}+\epsilon, \;\forall n\in \{0, \ldots, K\}\jm{.}
\end{align*}
\end{theorem}

The proof of this Theorem can be found in the \jm{Appendix}.

\subsection{A converse theorem}\label{A converse theorem}
So far we have shown that if the pdf of the generative model, $\tilde{p}$, is close to the target pdf, $p$, then the statistic $A_{K}$ is close to uniform. A natural question to ask is whether $A_{K}$ displaying a uniform distribution implies that $\tilde{p}=p$. In this \jm{section}, we prove that \jm{if $A_{K}$ is uniform, i.e., $\mathbb{Q}_{K}(n)=\dfrac{1}{K+1}$, for all $n\in \{0, \ldots, K\}$ and all $K$, then $p=\tilde{p}$ almost everywhere.}

To this end, we recall a key result from \cite{elvira2021performance}.

\begin{lemma}\label{lemma1}
Let $p$ and $\tilde{p}$ be two univariate pdf's with associated cfd's $F$ and $\tilde{F}$. If the rank statistic $A_{K}$ constructed in Section \ref{Uniform rank statistics} has \jm{a} uniform distribution on the set $\{0, \ldots, K\}$ then $(\tilde{F}^{n}, p) = \dfrac{1}{n+1}, \quad\forall n\in \{0, 1, \ldots, K\}$, where $\tilde{F}^n$ is the $n$-th power of $\tilde{F}$.
\end{lemma}
\begin{proof}
    See \cite[Theorem 2]{elvira2021performance}.
\end{proof}

Based on Lemma \ref{lemma1} we can prove the result anticipated at the beginning of this section.

\begin{theorem}\label{theorem3}
    Let $p$ and $\tilde{p}$ be pdf's of univariate real \jm{r.v.'s and let} $A_{K}$ be the rank statistic constructed in Section \ref{Discrete uniform rank statistics}. If $A_{K}$ has a discrete uniform distribution on $\{0, \ldots, K\}$ for every $K\in \mathbb{N}$, then $p=\tilde{p}$ almost everywhere.
\end{theorem}

\begin{proof}
    Let $Y$ be a r.v. with pdf $p$ and cdf $F$. From the inversion theorem (see, e.g., \cite[Theorem 2.1]{martino2018independent}) we obtain that $U=F(Y)$ is a standard uniform r.v., i.e., $U\sim \mathcal{U}(0, 1)$ and as consequence,
    \begin{align*}
        \mathbb{E}[U^{n}] = (F^{n}, p) = \dfrac{1}{n+1}, \quad \forall n \in \mathbb{N}.
    \end{align*}
    On the other hand, since $A_{K}$ is uniformly distributed on $\{0,\ldots, K\}$, Lemma \ref{lemma1} implies that the r.v. $\tilde{U}=\tilde{F}(Y)$ \jm{satisfies}
    \begin{align*}
        \mathbb{E}[\tilde{U}^{n}]= (\tilde{F}^{n}, p) =\dfrac{1}{n+1}, \quad \forall n\in \mathbb{N},
    \end{align*}
    \vspace{-3mm}
    hence,
    \begin{align}\label{eq1}
        \tilde{U} = \tilde{F}(Y)\sim \mathcal{U}(0, 1)
    \end{align}
    \jm{as well}. However, the relationship in Eq. \eqref{eq1} implies that $Y = \tilde{F}^{-1}(\tilde{U})$ has cdf $\tilde{F}$ (from the inversion theorem again), hence $F=\tilde{F}$, which implies that $p=\tilde{p}$ almost everywhere.
\end{proof}

\section{\jm{The invariant statistical loss function}}

\jm{Taken together, Theorems \ref{theorem1}, \ref{theorem2} and \ref{theorem3} imply that if we train a generative model with output pdf $\tilde p$ to make the random statistics $A_K$ uniform for sufficiently large $K$, then we can expect that $\tilde p \approx p$, where $p$ is the actual pdf of the available data that has been used for training. This is true independently of the form of the pdf $p$ and, in this sense, it enables us to construct loss functions which are themselves invariant w.r.t. $p$. In this section we introduce one such invariant loss and then propose a suitable surrogate which is differentiable w.r.t. the generative model parameters $\theta$ and, therefore, can be used for training using standard optimization tools.} 

\subsection{The Invariant Statistic Loss}\label{An Invariant-Statistic Loss}

The training data set consists of a set of $N$ i.i.d. samples $y_1,\ldots,y_N$ from the \jm{true} data distribution $p$. For \jm{each} $y_n$, along with $K$ i.i.d. samples from the generative model \jm{$\tilde{{\bf y}}=[\tilde y_1, \ldots, \tilde y_K]^\top$, where $\tilde y_i = g_{\theta}(z_i)$,} we can obtain one sample of the r.v. \(A_K\), that we denote as \(a_{K,n}\).

The great advantage of our method is that we simply have to verify the discrepancy between the discrete uniform distribution with $K+1$ outcomes and the empirical distribution of \jm{the set of statistics} \(a_{K,1}, a_{K,2}, \ldots, a_{K,N}\), i.e., their associated histogram. On the downside, note that, by definition, the construction of the statistic \(A_K\) in \jm{Eq.} \eqref{eq:AK} from samples of the generator does not \jm{yield} a function that is differentiable w.r.t. its parameters $\theta$.

We now introduce a new loss function, \jm{coined \textit{invariant statistical loss}} (ISL) that can be optimized w.r.t. $\theta$ and mimics the \jm{construction of a histogram from the statistics \(a_{K,1}, a_{K,2}, \ldots, a_{K,N}\)}. For this purpose, given a real \jm{data point} $y$, we can tally how many of the $K$ \jm{simulated samples in $\tilde {\bf y}$ are less than the observation $y$. Specifically, one computes}
\begin{align*}
\tilde{a}_K(y)=\sum_{i=1}^{K}\sigma_\alpha(y - \tilde{y}_{i}) = \sum_{i=1}^{K}\sigma_\alpha(y - g_{\theta}(z_{i})),
\end{align*}
\jm{where $z_i$ is a sample from a univariate standard Gaussian distribution, i.e., $z_i \sim p_0 \equiv \mathcal{N}(0,1)$, $\sigma(x):=1/\left(1 + \exp(-x) \right)$ is the sigmoid function and $\sigma_\alpha(x) :=\sigma(\alpha x)$. We remark that $\tilde a_K(y)$ is a differentiable (w.r.t. $\theta$) approximation of the actual statistic $A_K$ for the observation $y$. The parameter $\alpha$ enables us to adjust the slope of the sigmoid function to better approximate the (discrete) `counting' in the construction of $A_K$ (other alternatives can certainly be chosen here). Sharper functions offer better approximation but unstable gradients. In our case, $\alpha \in [10,20]$ has been effective in practice.}

To construct a differentiable \jm{surrogate} histogram from \jm{$\tilde{a}_K(y_1),\ldots,\tilde{a}_K(y_N)$}, we leverage a sequence of differentiable functions designed to mimic the bins around $k\in\{0,\ldots,K\}$, replacing sharp bin edges with functions that mirror bin values at $k$ and smoothly decay \final{outside a neighborhood of $k$}. In our particular case, we consider \jm{radial basis function (RBF)} kernels $\{\psi_{k}\}_{k=0}^{K}$ centered at $k\in \{0, \ldots,K\}$ with length-scale $\nu^2$, i.e., $\psi_k(a)=\exp(-(a-k)^2/2\nu^2)$.  Thus, the approximate normalized histogram count at bin $k$ is given by
\begin{align}
    q[k] = \frac{1}{N}\sum_{i=1}^{N}\psi_k(\tilde{a}_K(y_i)),
\end{align}
for $k=0,\ldots,K$. 
The \jm{proposed} ISL is \jm{now} computed as $\ell$-norm distance between the uniform pmf $\frac{1}{K+1}\mathbf{1}_{K+1}$ and the vector \jm{of empirical probabilities} $\mathbf{q}=\left[q[0], q[1],\ldots, q[K]\right]^\top$, \jm{namely,}
\begin{align}\label{eq:ISL}
    \mathcal{L}_{\final{ISL}}(\theta,K,\alpha,\nu) = \left|\left|\frac{1}{K+1}\mathbf{1}_{K+1}-\mathbf{q}\right|\right|_\ell.
\end{align}

\final{In Figure \ref{approximation error ISL}, we compare the surrogate ISL loss function and the theoretical counterpart for an exemplary case}. The set of hyperparameters is listed in Table \ref{table hyperparameters}. Note that evaluating \eqref{eq:ISL} requires \jm{$\mathcal{O}(NK)$} operations, but we can rely on mini-batch optimization to reduce it to \jm{$\mathcal{O}(M K)$}, being $M$ the mini-batch size. Algorithm \ref{Adaptative Block Learning Algorithm pseudocode} summarizes the training method.

\begin{figure}[h]
\centering
\includegraphics[width=0.8\linewidth]{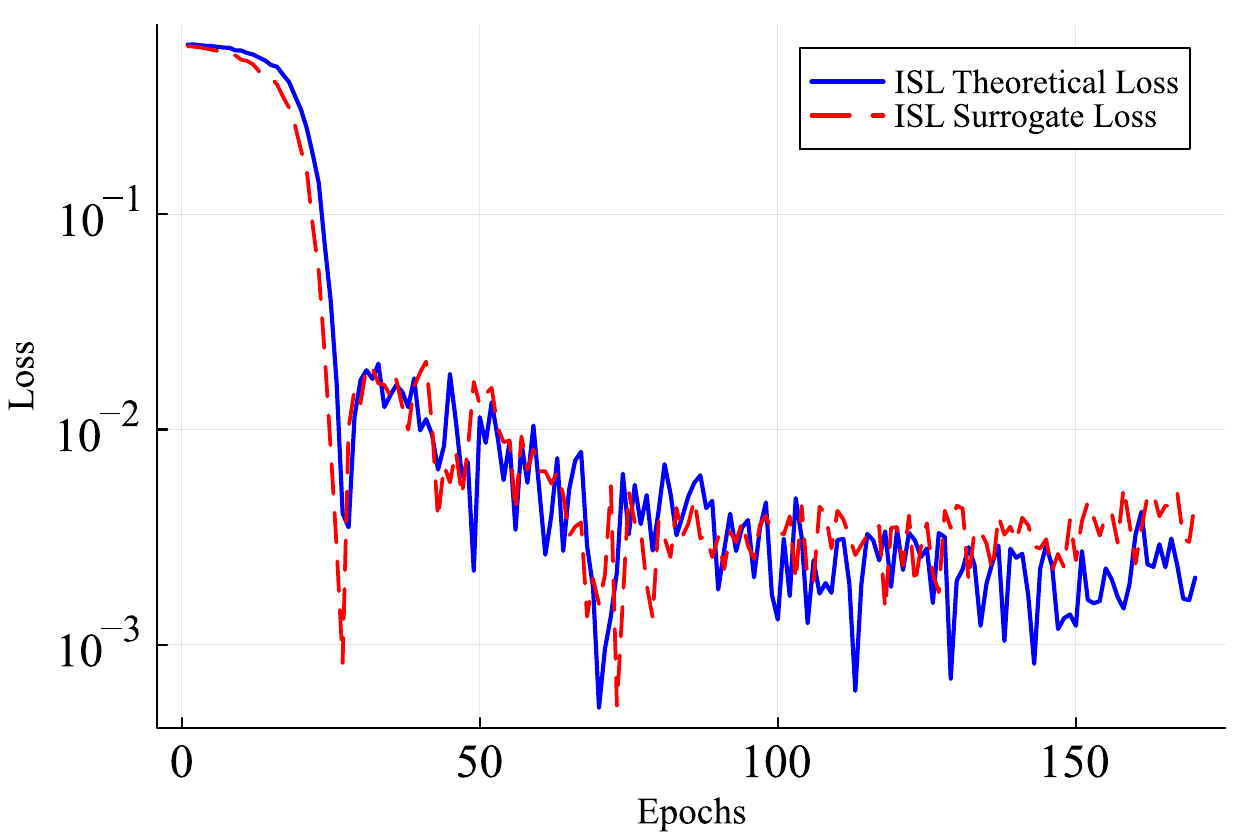}
\caption{\small ISL surrogate loss function vs. ISL theoretical loss during the training of $\mathcal{N}(4, 2)$, $K=1000$, $N=1000$, log-scale.} \label{approximation error ISL}
\end{figure}

\begin{table}[h]\label{table hyperparameters}
\small
\caption{\small ISL list of hyperparameters.}
\begin{center}\label{table hyperparameters}
\small 
\setlength{\tabcolsep}{3pt}
\begin{tabular}{|c||c|}
\hline
\textbf{Hyperparameter}  &\textbf{Value} 
\\
 \toprule
 Nº samples, $K$ &tunable\\
 Activation function & $\sigma(\alpha\cdot x)$\\
 $\{\psi\}_{k=0}^{K}$& RBF Kernels, length-scale=$\nu^2$ \\
 \bottomrule
\end{tabular}
\end{center}
\end{table}

\begin{algorithm}[t]
\setstretch{1.4}
\caption{ISL Algorithm}\label{Adaptative Block Learning Algorithm pseudocode}

\begin{algorithmic}[1]
\STATE \textbf{Input} Neural Network $g_{\theta}$; Hyperparameter $K$, $N$; Number of Epochs; Batch Size $M$; Training Data $\{y_i\}_{i=1}^N$.
\STATE \textbf{Output} Trained Neural Network $g_{\theta}$.
\STATE \textbf{For} $t=1,\ldots, \operatorname{epochs}$ \textbf{do}
\STATE \textbf{For} $\text{iteration}=1,\ldots, N/M$ \textbf{do}
\STATE \hspace{0.15in} Select $M$ samples from $\{y_{j}\}_{j=1}^{N}$ at random
\STATE \hspace{0.15in} $\{z_{i}\}^{K}_{i=1} \sim \normdist{0}{1}$
\STATE \hspace{0.15in} $\mathbf{q} = \frac{1}{M}\sum_{j=1}^{M}\psi_k\Big(\sum_{i=1}^{K}\sigma_{\alpha}(y_{j} - g_{\theta}(z_i))\Big)$

\STATE \hspace{0.15in} $\nabla_{\theta}\operatorname{loss} \leftarrow \nabla_{\theta}\left|\left|\frac{1}{K+1}\mathbf{1}_{K+1}-\mathbf{q}\right|\right|_2$ 
\STATE \hspace{0.15in} $\text{Backpropagation}(g_{\theta}, \nabla_{\theta}\operatorname{loss})$
\STATE \textbf{return} $g_{\theta}$
\vspace{0.15in}
\end{algorithmic}
\end{algorithm}

\subsection{Progressively Training by Increasing $K$}

The parameter \(K\) dictates the number of \jm{simulated} samples \jm{per data point}. As established in Section \ref{Uniform rank statistics}, larger values of \(K\) ensure \jm{a better approximation of the true data distribution, albeit at the expense of an increased computational effort.}

\jm{We have found that the training procedure can be made more efficient if it is performed in a sequence of increasing values of $K$, say $K^{(1)} < K^{(2)} < \cdots < K^{(I)}$, where $I$ is the total number of stages in the process and $K^{(I)} = K_{max}$ is the maximum admissible value of $K$ --a hyperparameter to be chosen depending on the available computing power. Hence, at the $i$-th stage we train the generative model using the ISL \final{$\mathcal{L}_{ISL}(\theta,K^{(i)})$}. During training, we periodically run a test (e.g., a Pearson $\chi^2$ test against the discrete uniform distribution using the statistics $a_{K,1}, a_{K,2}, \ldots, a_{K,N}$). When the hypothesis ``$A_K$ is uniform'' is accepted we increase $K^{(i)}$ to $K^{(i+1)}$ for the $(i+1)$-th stage and keep training. Note that the generative model parameters adjusted at the $i$-th stage serve as an initial condition for training at the $(i+1)$-th stage. We have followed this methodology for all the experiments presented in this paper.}

\section{Invariant Statistical Loss for Time Series Prediction} \label{Invariant Statistical Loss for Time Series Prediction}

So far we have focused our discussion on 1D random variables, but our approach is amenable to be used on (scalar \final{and multivariate}) time series.
Let $y[0], y[1],\ldots,y[T]$ be a sample of a discrete-time random process between time\footnote{We assume with no loss of generality that the origin of the process is at time $t=0$.} $0$ and time $T$, and let $Y[t]$ be the random variable of the process at time $t$, with unknown conditional distribution $p_t=p(Y[t]|Y[0], Y[1], \cdots, Y[t-1])$. 

Given the sequence $y[0], y[1],\ldots,y[t-1]$, our goal is to train an autorregressive conditional implicit generator network $g_\theta(z,\mathbf{h}[t])$ that produces samples 
whose distribution, referred to as $\tilde p_t$, is
close to the unknown pdf $p_t$ when $z$ is a sample drawn from a standard Gaussian. The new input to our generator NN, $\mathbf{h}[t]$, is an embedding of the sequence $y[0], y[1],\ldots,y[t-1]$ provided by a suitable neural network. Specifically, we consider a simple recurrent neural network (RNN)
connected to the generator NN as illustrated in Figure \ref{fig:my_tikz}.
During training, the newly-arrived observation at time $t$, $y[t]$, is fed into the RNN to collapse the process history into \textit{hidden} state $\mathbf{h}[t]$. The latter is then exploited 
(along with input noise $z$) by generator NN $g_\theta(z,\mathbf{h}[t])$ to predict the next observation 
$\tilde y[t+1]$. During test/validation, forecasting is performed by feeding $\tilde y[t+1]$ back to the RNN.

\tikzstyle{ts neural network style}=[rounded corners=4]
\tikzset{
    rnn cell/.style={rectangle, rounded corners=5, fill=blue!20, minimum width=1.5cm, minimum height=1.2cm}
}
\tikzset{
    gen cell/.style={rectangle, rounded corners=5, fill=red!20, minimum width=1.5cm, minimum height=1.2cm}
}
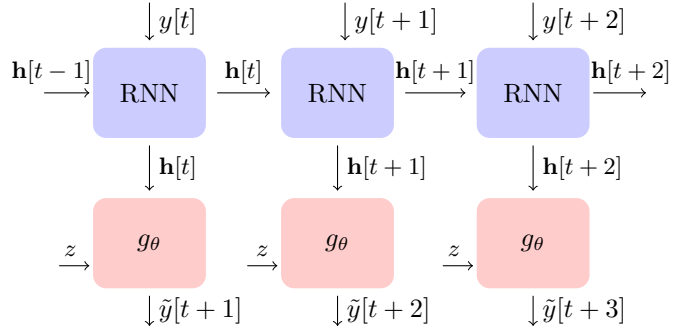
\begin{figure}[h] \label{architecture ts}
    \centering
\begin{tikzpicture}[ts neural network style]
    \node[rnn cell] at (1.4,3.5) {RNN};
    \node[gen cell] at (1.4,1.5) {$g_\theta$};

    \node[rnn cell] at (3.9,3.5) {RNN};
    \node[gen cell] at (3.9,1.5) {$g_\theta$};

    \node[rnn cell] at (6.5,3.5) {RNN};
    \node[gen cell] at (6.5,1.5) {$g_\theta$};
    
    \draw[->] (1.4,4.7) -- (1.4,4.2) node[midway, right] {$y[t]$};
    \draw[->] (1.4,2.8) -- (1.4,2.2) node[midway, right] {\footnotesize $\mathbf{h}[t]$};
    \draw[->] (1.4,0.8) -- (1.4,0.4) node[midway, right] {$\tilde{y}[t+1]$};

    \draw[->] (4.0,4.7) -- (4.0,4.2) node[midway, right] {$y[t+1]$};

    \draw[->] (6.5,4.7) -- (6.5,4.2) node[midway, right] {$y[t+2]$};
    \draw[->] (3.9,2.8) -- (3.9,2.2) node[midway, right] {\footnotesize $\mathbf{h}[t+1]$};
    \draw[->] (3.9,0.8) -- (3.9,0.4) node[midway, right] {$\tilde{y}[t+2]$};

    \draw[->] (6.5,2.8) -- (6.5,2.2) node[midway, right] {\footnotesize $\mathbf{h}[t+2]$};
    \draw[->] (6.5,0.8) -- (6.5,0.4) node[midway, right] {$\tilde{y}[t+3]$};

    \draw[->] (0.0,3.5) -- (0.6,3.5) node[midway, above] {\hspace{-4mm}\footnotesize $\mathbf{h}[t-1]$};
    \draw[->] (2.3,3.5) -- (3.0,3.5) node[midway, above] {\footnotesize $\mathbf{h}[t]$};
    \draw[->] (0.2,1.2) -- (0.6,1.2) node[midway, above] {\hspace{-1mm}\footnotesize $z$};

    \draw[->] (4.8,3.5) -- (5.6,3.5) node[midway, above] {\footnotesize $\mathbf{h}[t+1]$};
    \draw[->] (2.7,1.2) -- (3.1,1.2) node[midway, above] {\footnotesize $z$};

    \draw[->] (5.3,1.2) -- (5.7,1.2) node[midway, above] {\hspace{-1mm}\footnotesize $z$};

    \draw[->] (7.3,3.5) -- (8.0,3.5) node[midway, above] {\hspace{3mm}\footnotesize $\mathbf{h}[t+2]$};
\end{tikzpicture}
\caption{Conditional implicit generative model for time-series prediction.}
\label{fig:my_tikz}
\end{figure}

Notice that in this new temporal setup, all the results established in the previous sections remain valid. However, while before we had a single statistic $A_K$, we now have one at every time instant, $A_K[t]$. The latter is still constructed according to Eq. \eqref{eq:AK}, and will follow a uniform distribution if $\tilde p_t$ is close enough to $p_t$.
This result is invariant w.r.t. $t$ and $p_t$.

We use the sequence of observations $\mathbf{y}=[y[0], y[1],\ldots,y[T]]$ to construct a sequence of statistics $a_K[0], a_K[1], \ldots, a_K[T]$, whose aggregated empirical distribution should be approximately uniform if $\tilde{p}_t\approx p_t$ for $t=0,\ldots,T$. 

To construct the ISL, we apply a similar procedure to obtain a differentiable histogram surrogate. 

When observations are multivariate, i.e., we have a collection of time series\footnote{Assume for simplicity that they are of the same length $T$, but this is not a constraint here.}, $\{\mathbf{y}_i\}_{i=1}^N$, we can essentially apply the same procedure. Specifically, for every time series, $\mathbf{y}_i$, we obtain a sequence of statistics $\mathbf{a}_{i,K}= [a_{i,K}[0], a_{i,K}[1], \ldots, a_{i,K}[T]]$. If $\tilde{p}_t\approx p_t$ for $t=0,\ldots,T$, then again the agreggated empirical distribution of $\mathbf{a}_{1,K}\bigcup\mathbf{a}_{2,K}\bigcup\ldots\bigcup\mathbf{a}_{T,K}$, should be approximately uniform. The ISL surrogate requires in this general scenario $\mathcal{O}(TNK)$ operations to evaluate, which can be reduced to $\mathcal{O}(WMK)$ if we use mini-batches of $M$ data points and we consider a sliding window of length $W$. 

\begin{table*}[h]\label{Results ISL table}
\centering
\caption{\small Comparison of ISL results with vanilla GAN, WGAN, and MMD-GAN under the hyperparameters: Noise $\sim\mathcal{N}(0,1)$, $K_{max}=10$, $\operatorname{epochs}=1000$, and $N=1000$.} \label{Learning1D table}
\begin{center}
\small 
\setlength{\tabcolsep}{3pt}
\begin{tabular}{lc||rrr|rrr|rrr|rrr}
&\multicolumn{3}{r}{\textbf{ISL}}&\multicolumn{3}{r}{\textbf{GAN}}&\multicolumn{3}{r}{\textbf{WGAN}}&\multicolumn{3}{r}{\textbf{MMD-GAN}}&\\\\
\toprule
&\textbf{Target}  &\textbf{KSD} &\textbf{MAE} &\textbf{MSE}&\textbf{KSD}&\textbf{MAE}&\textbf{MSE}&\textbf{KSD}&\textbf{MAE}&\textbf{MSE}&\textbf{KSD}&\textbf{MAE}&\textbf{MSE}

\\
\midrule

&$\normdist{4}{2}$         &$\boldnum{8.5e^{-3}}$ & $\boldnum{0.1005}$ & $\boldnum{0.0348}$& $0.0154$& $0.4335$& $0.4015$&$0.0468$&$4.4644$&$55.1103$&$0.0238$&$4.2689$&$30.3614$\\
&$\mathcal{U}(-2,2)$    &$\boldnum{0.0125}$ & $0.1341$ & $0.0258$&$0.0265$ & $\boldnum{0.1097}$& $\boldnum{0.0180}$&$0.0790$&$0.5659$&$0.4025$&$0.0545$&$0.5322$&$0.4130$\\
&Cauchy(1, 2)   &$\boldnum{6.3e^{-3}}$ & $\boldnum{7.4957}$ & $8177$& $0.0823$&$8.9628$&$8207$&$0.0302$&$10.5681$&$\boldnum{8127}$&$0.0263$&$9.6798$&$57975$\\
&Pareto(1, 1)   &$\boldnum{0.0822}$ & $\boldnum{4.7742}$ & $7843$&$0.0987$&$12.6397$&$114970$&$0.4919$&$7.6435$&$7062$&$0.5008$&$9.023$&$10674$\\
&$\text{Model}_1 $  &$0.0169$ & $\boldnum{0.2995}$ & $\boldnum{0.1437}$ &$\boldnum{0.0116}$& $0.4246$& $0.4537$&$0.0315$&$1.8831$&$4.8943$&$0.0450$&$1.6433$&$4.6988$\\
&$\text{Model}_2$   &$\boldnum{0.0173}$ & $0.6660$ & $0.7757$&$0.0536$&$\boldnum{0.6072}$&$\boldnum{0.7506}$&$0.0459$&$4.2390$&$28.9843$&$0.0232$&$8.7218$&$112$\\
&$\text{Model}_3$   &$\boldnum{0.1591}$ & $0.8791$ & $\boldnum{1.5207}$&$0.1846$&$\boldnum{0.4447}$&$1.6053$&$0.2986$&$3.0376$&$13.1256$&$0.5608$&$3.2046$&$21.5244$\\
\bottomrule
\end{tabular}\label{Results ISL table}
\end{center}
\end{table*}

\begin{table}[h]\label{diffusion}
\centering
\caption{\small
\final{ISL vs Diffusion model (1D case).}}\label{diffusion}
\small
\setlength\tabcolsep{2pt} 
\begin{tabular}{|c|c|c|}
\multicolumn{1}{c}{}& \multicolumn{1}{c}{\textbf{ISL}} & \multicolumn{1}{c}{\textbf{Diff.}} \\
\hline
\textbf{Target} & \textbf{KSD} & \textbf{KSD}\\
\hline
$\mathcal{N}(4, 2)$ & \textbf{8.5e-3} & 0.0867 \\
$\mathcal{U}(-2, 2)$ & \textbf{0.0125} & 0.2307 \\
Cauchy(1, 2) & \textbf{6.3e-3} & 0.0226 \\
Pareto(1, 1) & 0.0822 & \textbf{0.0512} \\
$\text{Model}_1$ & \textbf{0.0169} & 0.1507 \\
$\text{Model}_2$ & \textbf{0.0173} & 0.2217 \\
$\text{Model}_3$ & 0.1591 & \textbf{0.05475} \\
\hline
\end{tabular}
\end{table}

\section{Related Work}

Among existing methods for training implicit generative models, most of the works in the literature consider the use of a discriminator/critic network that classifies whether data points are artificially generated or are real samples \cite{mohamed2016learning}. From basic GANs \cite{goodfellow2014generative}, to WGANs \cite{arjovsky2017wasserstein} and $f$-GANs \cite{nowozin2016f}, this approach
is known to suffer from mode collapse and training instabilities \cite{arora2017gans, arora2017generalization}. These issues can be alleviated with techniques such as mini-batch discrimination \cite{salimans2016improved} and spectral normalization \cite{miyato2018spectral}.

Among training methods for implicit models that do not require a discriminator, Maximum Mean Discrepancy (MMD) losses are a popular choice \cite{gretton2012kernel, li2015generative, gouttes2021probabilistic}. Incorporating MMD ensures that the distance between the real and generated data distributions is minimized in a Reproducing Kernel Hilbert Space, leading to more stable training and nuanced density approximations. However, due to the reformulation in terms of kernels, the complexity of evaluating the loss is large (scales quadratically with the amount of data), and approximations are needed. 

The literature on implicit temporal generative models is comparatively much shorter since the best-performing methods nowadays are based on prescribed conditional autoregressive models combined with sequential transformer architectures, conditional normalizing flows, or difussion models \cite{salinas2020deepar, li2019enhancing, moreno2023deep,rasul2020tempflow, pmlr-v139-rasul21a}.  To our knowledge, due to the inherent invariance, ISL is the first method to adapt implicit generative models to the temporal domain in a straightforward manner. \final{As we demonstrate, ISL achieves competitive performance when compared to prescribed autoregressive generative methods such as  Phrophet \cite{Sean18}, TRMF \cite{Yu2016}, Deep State Space Models (DSSM) \cite{Rangapuram18}, DeepAR \cite{salinas2020deepar}, and transformer architectures \cite{zeng2023transformers, wu2021autoformer, zhou2021informer, li2019enhancing} with simpler network structures in both univariated and multivariated scenarios.}

\section{Experiments}
Our experiments evaluate implicit models trained with the ISL in Eq. \eqref{eq:ISL} for density estimation. We showcase its qualitative advantage over classical generative models for independent time settings, capturing multimodal and heavy-tailed distributions. ISL outperforms vanilla-GAN, Wasserstein-GAN, and MMD-GAN baselines. \final{ISL also shows promise when compared to more modern methods such as diffusion models.} In addition, we explore its potential as a GAN regularizer. Our experimental settings for this section are grounded in the study by \cite{zaheer2017gan}. Subsequently, we assess ISL's effectiveness in time series forecasting using synthetic and real datasets.

In the supplementary material, we extend the experimental validation by comparing different NN hyperparameters and several configurations of the proposed experimental settings.

\subsection{Learning $1$-D distributions}

We start considering the same experimental setup as \cite{zaheer2017gan}.
In particular, we aim at approximating different target distributions using $1000$ i.i.d. samples. 


The final three rows describe mixture models, each with equal weighting among their components,

$\text{Model}_1$ consists of a mixture model that combines two normal distributions, $\normdist{5}{2}$ and $\normdist{-1}{1}$.  $\text{Model}_2$ is a mixture model that integrates three normal distributions $\normdist{5}{2}$, $\normdist{-1}{1}$ and $\normdist{-10}{3}$. Finally, $\text{Model}_3$ combines a normal distribution $\normdist{-5}{2}$ with a Pareto distribution $Pareto(5, 1)$.

As generator NN, we use a $4$-layer multilayer perceptron (MLP) with $7, 13, 7$ and $1$ units at the corresponding layers. In the supplementary material we compare the ISL performance with different architectures. As activation function we use a exponential linear unit (ELU). We train each setting up to $10^{4}$ epochs with $10^{-2}$ learning rate using Adam. We compare ISL, GAN, WGAN, and MMD-GAN using KSD (Kolmogorov-Smirnov Distance), MAE (Mean Absolute Error), and MSE (Mean Squared Error) error metrics, defined as follows

\begin{align*}
    &\operatorname{KSD} = \sup_{x}|F(x) - \tilde{F}(x) |,\\&\operatorname{MAE} = \int_{-\infty}^{\infty}|f(z)-f_{\theta}(z)|\,\final{p_{z}(z)dz},\\&\operatorname{MSE} = \int_{-\infty}^{\infty}\big(f(z)-f_{\theta}(z)\big)^{2}\,\final{p_{z}(z)dz},
\end{align*}
where we have denoted by $F$, $\tilde{F}$, and $F_{Z}$ the cdfs associated with the densities $p$, $\tilde{p}$, and $p_{Z}$ respectively.

\begin{table*}[h]\label{Result table ISL GAN regu}
\caption{\small ISL combined with a GAN pre-trained generator: Left - Results after applying ISL method; Right - GAN results before ISL method. Hyperparameters: Noise $\sim\mathcal{N}(0,1)$, $K_{max}=10$, $\operatorname{epochs}=1000$, $N=1000$.} \label{sample-table}
\begin{center}
\small 
\begin{tabular}{lc||rrr|rrr}\label{Result table ISL GAN regu}
\centering
  \multirow{2}{*}{} &\multicolumn{3}{r}{\textbf{GAN+ISL}} & \multicolumn{3}{r}{\textbf{GAN}}&\\
\toprule
&\textbf{Target}& \textbf{KSD} &\textbf{MAE} &\textbf{MSE} &\textbf{KSD} &\textbf{MAE} &\textbf{MSE} \\
\midrule
&$\normdist{23}{1}$ &$\boldnum{7.4e^{-3}}$ & $\boldnum{0.0881}$ &$\boldnum{0.0119}$ &$0.0170$&$0.1995$&$0.0549$\\
&$\mathcal{U}(22,24)$ &$\boldnum{0.0287}$ & $\boldnum{0.1092}$ & $\boldnum{0.0327}$&$0.4488$&$0.3942$&$0.2221$\\
&Cauchy(23, 1) &$\boldnum{9.43e^{-3}}$ & $\boldnum{6.9527}$ & $\boldnum{937.3329}$&$0.1278$&$34.8575$&$21154$\\
&Pareto(23, 1) &$\boldnum{0.0264}$ & $\boldnum{0.6328}$ & $\boldnum{0.8890}$&$0.0474$&$131.0215$&$21418626$\\
&$\text{Model}_1$ &$\boldnum{5.63e^{-3}}$ & $0.6328$ & $0.8890$&$0.0160$&$\boldnum{0.4165}$&$\boldnum{0.4584}$\\
&$\text{Model}_2$ &$\boldnum{0.0161}$ & $\boldnum{0.4607}$ & $\boldnum{0.3350}$&$0.0399$&$0.8484$&$0.9670$\\
&$\text{Model}_3$ &$\boldnum{0.0805}$ & $\boldnum{0.4507}$ & $\boldnum{0.3373}$&$0.1378$&$0.7628$&$1.2044$\\
\bottomrule
\end{tabular}
\end{center}
\end{table*}

Observe in Table \ref{Results ISL table} that ISL outperforms the rest of methods in most scenarios, especially in the case of mixture models. Referring to \cite[Theorem 1]{zaheer2017gan}, it asserts that for a 1D distribution, there are at most two continuous transformations \(f\) such that if \(Z\) follows \final{\(p_{0}\)} and \(X\) follows \(p\), then \(f(Z)=X\). These transformations can be derived using the probability integral transform. In Figure \ref{fancy figure} we compare the generator function $g_\theta(z)$ with the true transformation function $f(z)$ for ISL (a) and MMD-GAN (b) for $\text{Model}_3$.

In Table \ref{diffusion} we compare ISL vs Diffusion model (1D case). The latter is based on an MLP with three layers of 100 units each, integrated with 100-dimensional positional encodings for time steps. It has been optimized for 20 epochs using MSE as a loss function, and linearly spaced betas (0.0004 to 0.06) across 100 denoising steps. For further information and details about parameters, refer to \cite{scotta2023understanding}.

\subsection{GAN pre-training}\label{subsec: GAN regularizer}

Our method can be used in combination with a GAN pre-trained generator. While GANs often struggle with multimodal distributions, typically capturing only some modes \cite{arora2017generalization,arora2017gans}, our approach fully represents the support of the true distribution. If not, disparities would arise between fictional and real data. Especially with a large \(K\), missed modes would cause fictional samples to diverge from real data modes, biasing the generated histogram. In this subsection, we now assess the ability of ISL to correct a potentially biased GAN construction of the generator $g_\theta(z)$. 

\final{Pre-training ISL with a GAN could potentially improve results if practitioners notice that ISL is suffering from a vanishing gradient problem due to poor parameter initialization. This strategy has only been used for the experiments in Table \ref{Result table ISL GAN regu}. For the rest of the 1D density approximations and time series forecasting experiments, no pre-training was applied.}

We use the same parameters as before for the generator. For the GAN critics, following \cite{zaheer2017gan}, we use a 4-layer MLP with 11, 29, 11, and 1 units. The ratio of updating critics and generators are searched in $\{2:1, 3:1, 4:1, 5:1\}$. We train each setting for $10^{4}$ epochs.

Upon training the GAN, we train the generator using ISL with $K_{max}=10$, 1000 epochs, and  $N=1000$ samples. The results are displayed in Table \ref{Result table ISL GAN regu} and they can directly compare these results with the ones obtained in \cite{zaheer2017gan}. The comparison w.r.t. Table \ref{Results ISL table} shows that pre-training with a GAN can improve the performance of the ISL method.

\subsection{Time Series}
Experiments are performed using both synthetic and real-world datasets to demonstrate the enhanced forecasting capabilities of the proposed method.

\subsubsection{Synthetic Time Series}
To begin, we conduct synthetic experiments to illustrate our method's proficiency in grasping the intrinsic probability distribution of the time series. We delve into the scenario involving autoregressive models with diverse parameters. 

Consider a simple autoregressive  process $\operatorname{AR}(p)$ of the form $X_t = \sum_{i=1}^{p}\phi_{i}X_{t-i} + \xi_{t}$.  In Table \ref{ARP table} we consider different parameter settings for $\phi_i$ and the noise variance, $\xi$. We display the forecasting performance of both temporal ISL compared to DeepAR \cite{salinas2020deepar} when the forecasting starts upon observing $\tau_0$ samples of the process. As error metrics, we consider the Normal Deviation (ND), Root-mean-square deviation (RMSE), and Quantile Loss ($QL_{\rho=0.9}$) metrics (For further details, please refer to \cite{moreno2023deep}).

In this scenario, we have used for both methods (ISL and DeepAR) a two-layer RNN with $10$ units each, utilizing a rectified linear unit (RELU) as activation function. This network will be connected with an MLP, consisting of two layers with $16$ units each, applying RELU and identity activation functions. The learning rates will be ascertained within the range of $\{10^{-2}, 10^{-3}, 10^{-4}\}$ for the Adams optimizer. We train each model on $200$ signals of $1000$ elements each. 

\begin{table}[h]
\caption{Forecasting an Autoregressive Model AR(p), with noise $\xi=\normdist{0}{0.01}$.} \label{ARP table}
\begin{center}
\setlength{\tabcolsep}{2pt}
\scriptsize 
\begin{tabular}{|l|c|c||lll|}
\toprule
\textbf{Meths}&$\tau_0$& $\{\phi_i\}_i^p$ &\textbf{ND} &\textbf{RMSE} &\textbf{$\mathbf{QL_{\rho=0.9}}$}\\
\midrule \multirow{4}{*}{ISL}
&   $20$  &    $.5,.3,.2$ & $\mathbf{0.762 \pm 0.37}$ &  $\mathbf{3.788 \pm 1.84}$ &  $\mathbf{0.718 \pm 0.69}$\\
&    &   $.5,.2$ &  $\mathbf{0.104\pm 0.03}$ &   $\mathbf{0.551\pm 0.14}$ &   $\mathbf{0.174 \pm 0.06}$\\
\cline{2-6}
&  $50$  &   $.5,.3,.2$ & $\mathbf{0.850\pm 0.23}$ &   $\mathbf{6.553\pm 1.94}$ & $0.660\pm0.60$\\
&&$.5,.2$     &$\mathbf{0.155\pm0.09}$ & $\mathbf{1.206\pm0.64}$ & $\mathbf{0.039\pm0.02}$\\
\midrule

\hline \multirow{4}{*}{DeepAR}
&$20$&$.5,.3,.2$ &$1.031 \pm 0.05$ & $4.970 \pm 0.49$ & $1.269 \pm 0.80$\\
&&$.5,.2$ &$1.009\pm 0.03$ & $5.537\pm 0.34$ & $1.041 \pm 0.49$\\
\cline{2-6}
&$50$&$.5,.3,.2$ &$1.000\pm 0.01$ & $7.630\pm 0.61$ & $\mathbf{0.650\pm0.59}$\\
&&$.5,.2$     &$1.006\pm0.02$ & $8.811\pm0.35$ & $1.012\pm0.34$\\
\bottomrule
\end{tabular}
\end{center}
\end{table}

\subsubsection{Real World datasets}

We evaluated our model using several real-world datasets. The `electricity-f' dataset captures 15-minute electricity consumption of 370 customers \cite{misc_electricityloaddiagrams20112014_321}. The `electricity-c' dataset aggregates this data into hourly intervals. 

Training employed 2012-2013 data, and validation used 20 months starting from 15-April 2014. For this purpose, we have used a $2$-layer RNN layers with $3$ units each and a $2$-layer MLP with $10$ units each and a ReLU activation function. In Figure \ref{fancy figure:sub_c} we show the ISL 7-day forecast of one of the instances of the `electricity-c' dataset. In Tables \ref{table Real World dataset}, \ref{table Real World dataset 2} we compare the ISL performance with ARIMA \cite{wilsonbook}, Phrophet \cite{Sean18}, TRMF \cite{Yu2016}, a RNN based State
Space Model, DSSM \cite{Rangapuram18}, a transformer-based model (ConvTrans) \cite{convtrans}, and DeepAR \cite{salinas2020deepar}. As reported, with a relatively simple NN structure, ISL outperforms all baselines in all cases, except for ConvTrans for the 1-day forecast using the metric QL$_{\rho=0.5}$. Note, however that ConvTrans is much more complex than the proposed model, in terms of the NN structure and the number of parameters. This result proves that the inherent flexibility of the implicit model can compete with more complex prescribed models, usually applied in time-series forecasting. \final{In Table \ref{table univariate}, we make a comparison of ISL with different state-of-the-art transformer architectures on various databases (see \cite{zeng2023transformers} for details). In this case, we have used a 1-layer RNN with 5 units, followed by a Batch Normalization layer, and a 2-layer MLP with 10 units in each layer. The MLP uses a ReLU activation function in the first layer, an identity activation function in the last layer, and a $5\%$ dropout rate in the first layer. We observe that with a very simple architecture, ISL is competitive and often capable of surpassing transformer architectures for time series prediction.}

\begin{table}[h] \label{table Real World dataset}
\caption{\small Evaluation summary, using $\mathbf{QL_{\rho=0.5}/QL_{\rho=0.9}}$ metrics, on electricity-$c$ datasets, forecast window 1 and 7 days.} \label{Learning1D table}
\begin{center}
\small 
\setlength{\tabcolsep}{3pt}
\begin{tabular}{|c||ll|}
\hline
\textbf{Method}  &\textbf{electricity-$\text{c}_{1d}$} &\textbf{electricity-$\text{c}_{7d}$} 
\\
 \toprule
 ARIMA&$0.154/0.102$&$0.283/0.109$\\
 Prophet&$0.108/0.099$&$0.160/0.154$\\
 ETS&$0.101/0.077$&$0.121/0.101$\\
 TRMF&$0.084/-$&$0.087/-$\\
 DSSM&$0.083/0.056$&$0.085/0.052$\\
 DeepAR&$0.075/0.040$&$0.082/0.053$\\
 ConvTras&$\boldnum{0.059}/0.034$&$0.076/0.037$\\
 ISL&$0.062/\boldnum{0.0189}$&$\boldnum{0.063/0.021}$\\
 \bottomrule
\end{tabular}\label{table Real World dataset}
\end{center}\label{table Real World dataset}
\end{table}

\begin{table}[ht] \label{table Real World dataset 2}
\caption{\small Evaluation summary, using $\mathbf{QL_{\rho=0.5}/QL_{\rho=0.9}}$ metrics, on electricity-$f$ datasets, forecast window 1 day.} \label{Learning1D table}
\begin{center}
\small 
\setlength{\tabcolsep}{3pt}
\begin{tabular}{|c||c|}
\hline
\textbf{Method}  &\textbf{electricity-$\text{f}_{1d}$} 
\\
 \toprule
 TRMF&$0.094/-$\\
 DeepAR&$0.082/0.063$\\
 \final{ConvTrans}&$0.074/0.042$\\
 ISL&$\boldnum{0.050/0.036}$\\
 \bottomrule
\end{tabular}\label{table Real World dataset 2}
\end{center}\label{table Real World dataset 2}
\end{table}

\begin{table}[ht]
\caption{\small Univariate long sequence time-series forecasting results on four datasets.}
\centering
\scriptsize
\setlength\tabcolsep{2pt}\label{table univariate}
\begin{tabular}{|c|c|ll|ll|ll|ll|}
\hline
\rule{0pt}{2.5ex}
\textbf{DB} & \textbf{$\tau$} & \multicolumn{2}{c|}{\textbf{Autoformer}} & \multicolumn{2}{c|}{\textbf{Informer}} & \multicolumn{2}{c|}{\textbf{LogTrans}} & \multicolumn{2}{c|}{\textbf{ISL}} \\
& & MSE & MAE & MSE & MAE & MSE & MAE & MSE & MAE \\ \hline

\multirow{4}{*}{ETTh1}
&$96$  & $\boldnum{0.071}$ & $\boldnum{0.206}$ & $0.193$ & $0.377$ & $0.283$ & $0.468$ & $0.113$& $0.239$\\
&$192$ & $\boldnum{0.114}$ & $\boldnum{0.262}$ & $0.217$ & $0.395$ & $0.234$ & $0.409$ & $0.168$& $0.294$\\
&$336$ & $\boldnum{0.107}$ & $\boldnum{0.258}$ & $0.202$ & $0.381$ & $0.386$ & $0.546$ & $0.208$& $0.333$\\
&$720$ & $\boldnum{0.126}$ & $\boldnum{0.283}$ & $0.183$ & $0.355$ & $0.475$ & $0.629$ & $0.234$ & $0.369$\\ \hline

\multirow{4}{*}{ETTh2}
&$96$  & $\boldnum{0.153}$ & $\boldnum{0.306}$ & $0.213$ & $0.373$ & $0.217$ & $0.379$ & $0.184$& $0.335$\\
&$192$ & $0.204$ & $0.351$ & $0.227$ & $0.387$ & $0.281$ & $0.429$ & $\boldnum{0.177}$ & $0.333$\\
&$336$   & $0.246$ & $0.389$ & $0.242$ & $0.401$ & $0.293$ & $0.437$ & $\boldnum{0.175}$ &$\boldnum{0.334}$\\
&$720$ & $0.268$ & $0.409$ & $0.291$ & $0.439$ & $0.218$ & $0.387$ & $\boldnum{0.181}$& $\boldnum{0.342}$\\ \hline

\multirow{4}{*}{ETTm1}
&$96$ & $\boldnum{0.056}$ & $0.183$ & $0.109$ & $0.277$ & $0.029$ & $\boldnum{0.171}$ & $0.065$& $0.208$\\
&$192$ & $\boldnum{0.081}$ & $\boldnum{0.216}$ & $0.151$ & $0.310$ & $0.157$ & $0.317$ & $0.146$& $0.276$\\
&$336$ & $\boldnum{0.076}$ & $\boldnum{0.218}$ & $0.427$ & $0.591$ & $0.289$ & $0.459$ & $0.107$& $0.242$\\
&$720$ & $0.110$ & $0.267$ & $0.438$ & $0.586$ & $0.430$ & $0.579$ & $\boldnum{0.092}$ & $\boldnum{0.224}$\\ \hline

\multirow{4}{*}{Exch.}
&$96$  & $0.241$ & $0.387$ & $0.591$ & $0.615$ & $0.279$ & $0.441$& $\boldnum{0.193}$& $\boldnum{0.298}$\\
&$192$ & $0.273$ & $0.403$& $1.183$ & $0.912$ & $1.950$ & $1.048$ &$\boldnum{0.198}$&$\boldnum{0.300}$\\
&$336$ & $0.508$ & $0.539$ & $1.367$ & $0.984$ & $2.438$ & $1.262$ &$\boldnum{0.201}$&$\boldnum{0.301}$\\
&$720$ & $1.111$ & $0.860$ & $1.872$  & $1.072$ & $2.010$ & $1.247$ &$\boldnum{0.201}$&$\boldnum{0.299}$\\ \hline
\end{tabular}
\end{table}

\subsubsection{Multivariate time series}
\final{We have extended our model to incorporate multivariate scenarios. As described in Section \ref{Invariant Statistical Loss for Time Series Prediction}, our methodology entails independently fitting the marginal distribution for each respective output and training the average ISL. Comprehensive results of this approach, comparing ISL versus transformer techniques in multidimensional cases, are presented in Table \ref{table multivariate}. In this case, we have use the same architecture as the one described for the results of the Table \ref{table univariate}.}

\begin{table}[ht]\label{table multivariate}
\caption{\small  Multivariate long sequence time-series forecasting results on four datasets.}
\scriptsize
\centering
\setlength\tabcolsep{2pt} \label{table multivariate}
\begin{tabular}{|c|c|ll|ll|ll|ll|}
\hline
\rule{0pt}{2.5ex}
\textbf{DB}& \textbf{$\tau$} & \multicolumn{2}{c|}{\textbf{Autoformer}} & \multicolumn{2}{c|}{\textbf{Informer}} & \multicolumn{2}{c|}{\textbf{LogTrans}} & \multicolumn{2}{c|}{\textbf{ISL}} \\
&  & MSE & MAE & MSE & MAE & MSE & MAE & MSE & MAE \\ \hline
\multirow{4}{*}{ETTh1}
&$96$ & $0.449$ & $0.459$ & $0.865$ & $0.713$ & $0.878$ &$0.740$ & $\boldnum{0.384}$ & $\boldnum{0.459}$ \\
& $192$ & $0.500$ & $\boldnum{0.482}$ & $1.008$ & $0.792$ & $1.037$ & $0.824$ & $\boldnum{0.380}$ & $0.511$ \\
& $336$ &$0.521$ & $0.496$ & $1.107$ & $0.809$ &  $1.238$& $0.932$ & $\boldnum{0.334}$ & $\boldnum{0.392}$ \\
& $720$ &$\boldnum{0.514}$ & $\boldnum{0.512}$ & $1.181$ & $0.865$ & $1.135$ & $0.852$ & $0.747$& $0.667$\\\hline
\multirow{4}{*}{ETTh2}
&$96$ & $0.358$ & $\boldnum{0.397}$ & $3.755$ & $1.525$ &  $2.116$ &$1.197$ & $\boldnum{0.305}$ & $0.435$ \\
& $192$ & $0.456$ & $\boldnum{0.452}$ & $5.602$ & $1.931$ & $4.315$ &  $1.635$ & $\boldnum{0.449}$ & $0.525$ \\
& $336$ & $0.482$ & $\boldnum{0.486}$ & $4.721$ & $1.835$ &  $1.124$ & $1.604$ & $\boldnum{0.410}$ & $0.511$ \\
& $720$ & $0.515$ & $0.511$ & $3.647$ & $1.625$ & $3.188$& $1.540$ & $\boldnum{0.400}$ & $\boldnum{0.499}$ \\\hline
\multirow{4}{*}{ETTm1}
&$96$  & $0.672$ & $0.571$ & $0.543$ & $0.510$ & $0.600$ & $0.546$ & $\boldnum{0.408}$& $\boldnum{0.508}$\\
&$192$ & $0.795$ & $0.669$ & $\boldnum{0.557}$ & $\boldnum{0.537}$ & $0.837$ & $0.700$ & $0.573$& $0.623$\\
&$336$ & $1.212$ & $0.871$ & $0.754$ & $\boldnum{0.655}$ & $1.124$ & $0.832$ & $\boldnum{0.702}$& $0.670$\\
&$720$ & $1.166$ & $0.823$ & $0.908$ & $0.724$ & $1.153$ & $0.820$ & $\boldnum{0.768}$ & $\boldnum{0.686}$\\ \hline
\multirow{4}{*}{ETTm2}
& $96$ & $\boldnum{0.255}$ & $\boldnum{0.339}$ & $0.365$& $0.453$ & $0.768$ & $0.642$ & $0.276$ & $0.443$\\
& $192$ & $\boldnum{0.281}$ & $\boldnum{0.340}$ & $0.533$ & $0.563$  & $0.989$ & $0.757$ & $0.335$ & $0.463$ \\
& $336$ & $0.339$ & $\boldnum{0.372}$ & $1.363$ & $0.887$ & $1.334$ & $0.872$ & $\boldnum{0.305}$ & $0.455$\\
& $720$ & $0.433$ & $\boldnum{0.432}$ & $3.379$ & $1.338$  & $3.048$ & $1.328$ & $\boldnum{0.427}$ & $0.546$ \\\hline
\end{tabular}
\end{table}

\final{As we observe, with an RNN-based structure and a simple MLP implicit model, ISL ($\approx 25K$ parameters) achieves forecasting accuracy better than many state-of-the-art models, which are based on much more complex structures with millions of parameters such as transformers \cite{zeng2023transformers}.}

\section{Conclusion and Future Work}

We have introduced a novel criterion to train univariate implicit generator functions that simply requires checking uniformity on the statistic defined by Eq. \eqref{eq:AK}. We designed a \textit{surrogate} loss function to optimize the neural network models. In experiments, we showcased the method's ability to capture complex 1D distributions, including multimodal cases and heavy tails, distinguishing it from traditional generative models that frequently face challenges such as mode collapse. Additionally, our model demonstrated effectiveness in handling \final{univariate and multivariate} time series data with complex underlying density functions.

Several potential research paths exist. \final{A straightforward direction is the use of random projections to advance ISL capabilities for generating multidimensional data, such as images \cite{deshpande2018generative}. 
Additionally, a significant research trajectory would involve accurately establishing the order of convergence of the method in relation with its parameters \(K\) and \(N\).}
    
\section*{Acknowledgments}

This work has been partially supported by the the Office of Naval Research (award N00014-22-1-2647) and Spain's {\em Agencia Estatal de Investigaci\'on} (refs. PID2021-125159NB-I00 TYCHE and PID2021-123182OB-I00 EPiCENTER). Pablo M.
Olmos also acknowledges the support by the Comunidad de Madrid under grants
IND2022/TIC-23550 and ELLIS Unit Madrid.

\nocite{*}
\bibliographystyle{apalike}
\bibliography{ref1}

\hfill
\cleardoublepage

\setcounter{section}{0}
\setcounter{figure}{0}
\setcounter{table}{0}

\onecolumn
\aistatstitle{Appendix}
\tableofcontents
\vfill
\unhidefromtoc
\section{MISSING PROOFS}

For $f:\mathbb{R} \mapsto \mathbb{R}$ a real function and $p$ the pdf of an univariate real r.v. let us denote,
\begin{align*}
    (f,p) := \int_{-\infty}^{\infty} f(x)p(x) dx
\end{align*}
and let $B_{1}(\mathbb{R}) :=\Big\{(f:\mathbb{R}\to \mathbb{R}) : \sup_{x\in \mathbb{R}}|f(x)|\leq 1\Big\}$ be the set of real functions bounded by $1$.

Let us also denote the indicator function as
\begin{align*}
    I_{A}(x)=
    \begin{cases}
      1, & \text{if } x\in A, \\
     0, & \text{if } x\not\in A,
   \end{cases}
\end{align*}
for $A\subset\mathbb{R}$ and $x\in \mathbb{R}$.

\begin{theorem}\label{theorem2}
If $\sup_{f\in B_{1}(\mathbb{R})}|(f,p)-(f,\tilde{p})|\leq \epsilon$ then,
\begin{align*}
    \dfrac{1}{K+1}-\epsilon \leq \mathbb{Q}_{K}(n)\leq \dfrac{1}{K+1}+\epsilon, \;\forall n\in \{0, \ldots, K\}.
\end{align*}
\end{theorem}
\begin{proof}

Choose $y_{0}\in \mathbb{R}$. The univariative distribution function (cdf) asociated to the pdf $p$ is

\begin{align*}
    F(y_{0}) = \int_{-\infty}^{y_{0}}p(y)dy = \big( I_{(-\infty, y_{0})}, p\big),
\end{align*}

and for $\tilde{p}$ we obtain the cdf

\begin{align*}
    \tilde{F}(y_{0}) = \int_{-\infty}^{y_{0}}\tilde{p}(y)dy =  \big( I_{(-\infty, y_{0})}, \tilde{p}\big).
\end{align*}

Now, the rank statistic can be interpreted in terms of a binomial r.v. To be specific, choose some $y_{0}\in \mathbb{R}$ and then run $K$ Bernoulli trials where, for the $i$-th trial, we draw $\tilde{y}_{i}\sim \tilde{p}$ and declare a success if, and only if,  $\tilde{y}_{i}<y_{0}$. It is apparent that the probability of success for these trials is $\tilde{F}(y_{0})$ and the probability of having exactly $n$ successes is

\begin{align*}
    h_{n}(y_{0}) = \binom{K}{n} (\tilde{F}(y_{0}))^{n}(1-\tilde{F}(y_{0}))^{K-n}.
\end{align*}

If we now let $y_{0}\sim p$ and integrate $h_{n}(y_{0})$ with respect to $p(y)$ we obtain the probability of the event $A_{K}=n$, i.e.
\begin{align*}
    \mathbb{Q}_{K}(n) = (h_{n}, p), \quad\forall n\in \{0,\ldots, K\}.
\end{align*}

Now, since $h_{n}\in B_{1}(\mathbb{R})$, the assumption in the statement of Theorem $2$ yields

\begin{align*}
    |(h_{n},\tilde{p})- (h_{n},p)|\leq \epsilon,
\end{align*}
which, in turn, implies

\begin{align*}
    (h_{n}, \tilde{p})-\epsilon \leq (h_{n}, p) \leq (h_{n}, \tilde{p})+\epsilon.
\end{align*}
However, $(h_{n}, p)=\mathbb{Q}_{K}(n)$ by construction. On the other hand, if we let $y_{0}\sim\tilde{p}$ and integrate with respect to $\tilde{p}$ Theorem $1$ yields that $(h_{n}, \tilde{p}) = \dfrac{1}{K+1}$, hence we obtain the inequality
\begin{align*}
    \dfrac{1}{K+1} -\epsilon \leq \mathbb{Q}_{K}(n) \leq \dfrac{1}{K+1} +\epsilon,
\end{align*}
and conclude the proof.
\end{proof}



\clearpage
\newpage
\section{ADDITIONAL EXPERIMENTS}

In this section of the appendix, we expand upon the experiments conducted. In the first part, we will study how the results of the ISL method for learning 1-D distributions change depending on the chosen hyperparameters. Later, in a second part, we will examine how these evolve over time. Finally, we will extend the experiments conducted with respect to time series, considering a mixture time series case. We will also provide additional results for the `electricity-c' and `electricity-f' series, and lastly, we will include the results obtained for a new time series.

\subsection{The Effect of Parameter $K_{max}$}
We will first analyze the effects of the evolution of the hyperparameter $K_{max}$ for different target distributions. For training, the number of epochs is $1000$, the learning rate $10^{-2}$, and the weights are updated using Adam optimizer. Also, the number of (ground-truth) observations is, in every case, $N=1000$. First, we consider a standard normal distribution,  $\normdist{0}{1}$, as the source from which we sample random noise to generate the synthetic data. As generator, we use a $4$-layer multilayer perceptron (MLP) with $7, 13, 7$ and $1$ units at the corresponding layers. As activation function we use a exponential linear unit (ELU). We specify the target distribution and $K_{max}$ in each subcaption.

At the sight of the results (see Figure \ref{figure 1 appendix} and \ref{figure 2 appendix}) we observe that a higher value of $K_{max}$ generally leads to more accurate results compared to the true distribution. However, this is not always the case because, as demonstrated in the example of the Pareto distribution, the value of $K$ achieved during training is lower (specifically $6$) than the maximum limit set for K. In such cases, $K_{max}$ does not have an effect.

\begin{figure}[p]
    \centering
    
    \begin{minipage}{0.22\textwidth}
        \centering
        \includegraphics[width=\linewidth]{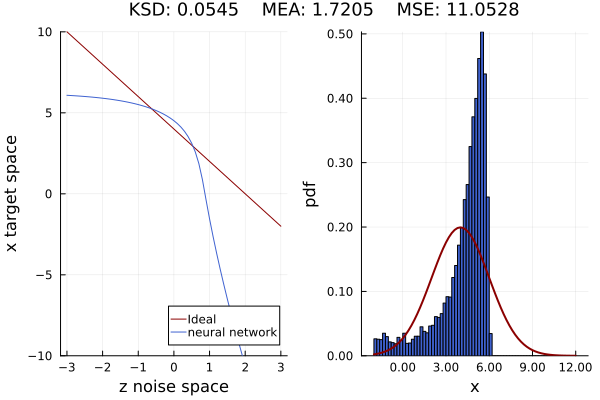}
        \caption*{\tiny $\mathcal{N}(4,2)$, $K_{max}=2$}
    \end{minipage}
    \hfill
    \begin{minipage}{0.22\textwidth}
        \centering
        \includegraphics[width=\linewidth]{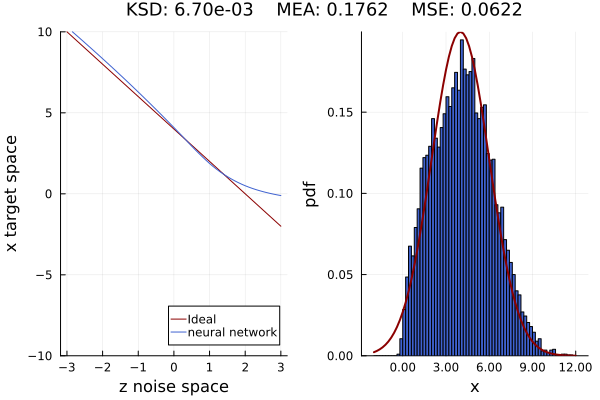}
        \caption*{\tiny $\mathcal{N}(4,2)$, $K_{max}=5$}
    \end{minipage}
    \hfill
    \begin{minipage}{0.22\textwidth}
        \centering
        \includegraphics[width=\linewidth]{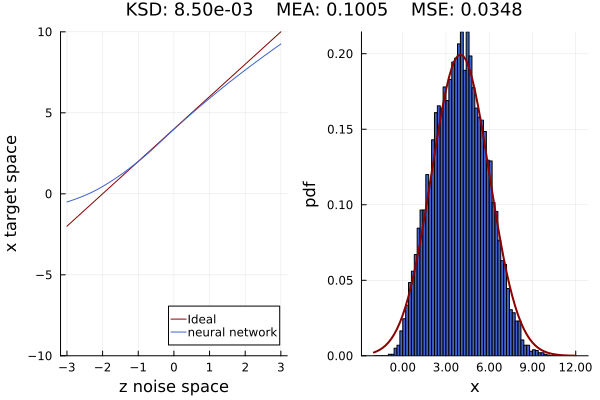}
        \caption*{\tiny $\mathcal{N}(4,2)$, $K_{max}=10$}
    \end{minipage}
    \hfill
    \begin{minipage}{0.22\textwidth}
        \centering
        \includegraphics[width=\linewidth]{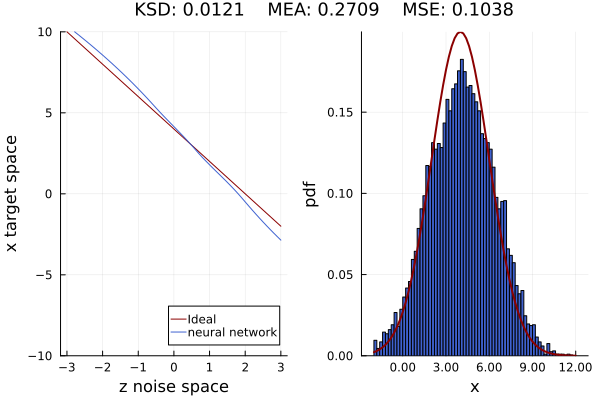}
        \caption*{\tiny $\mathcal{N}(4,2)$, $K_{max}=20$}
    \end{minipage}
    
    \vspace{2cm} 
    
    \begin{minipage}{0.22\textwidth}
        \centering
        \includegraphics[width=\linewidth]{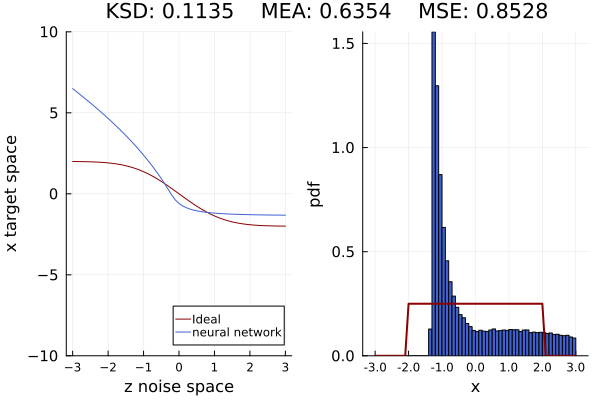}
        \caption*{\tiny $\mathcal{U}(-2,2)$, $K_{max}=2$}
    \end{minipage}
    \hfill
    \begin{minipage}{0.22\textwidth}
        \centering
        \includegraphics[width=\linewidth]{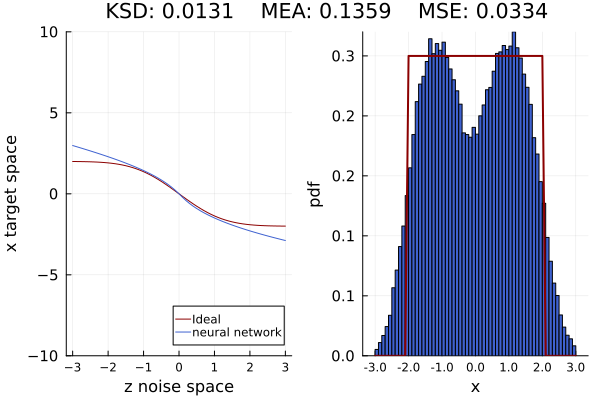}
        \caption*{\tiny $\mathcal{U}(-2,2)$, $K_{max}=5$}
    \end{minipage}
    \hfill
    \begin{minipage}{0.22\textwidth}
        \centering
        \includegraphics[width=\linewidth]{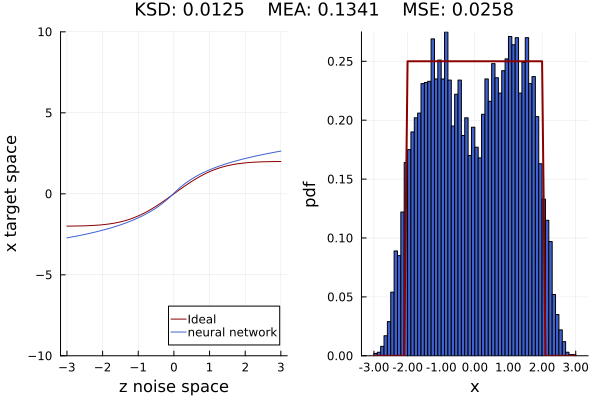}
        \caption*{\tiny $\mathcal{U}(-2,2)$, $K_{max}=10$}
    \end{minipage}
    \hfill
    \begin{minipage}{0.22\textwidth}
        \centering
        \includegraphics[width=\linewidth]{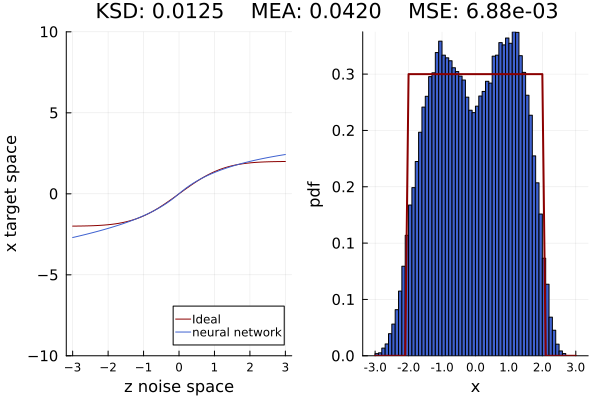}
        \caption*{\tiny $\mathcal{U}(-2,2)$, $K_{max}=20$}
    \end{minipage}
    
    \vspace{2cm} 
    
    \begin{minipage}{0.22\textwidth}
        \centering
        \includegraphics[width=\linewidth]{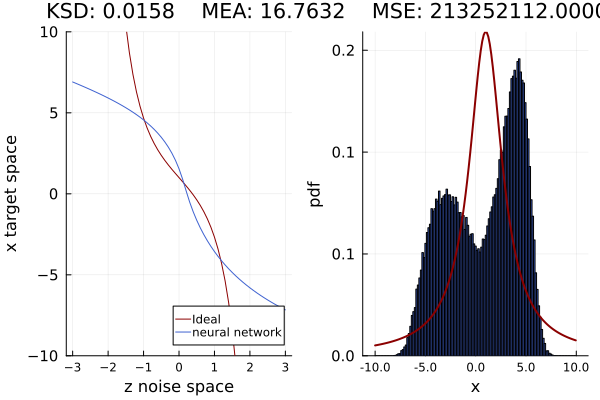}
        \caption*{\tiny Cauchy$(1,2)$, $K_{max}=2$}
    \end{minipage}
    \hfill
    \begin{minipage}{0.22\textwidth}
        \centering
        \includegraphics[width=\linewidth]{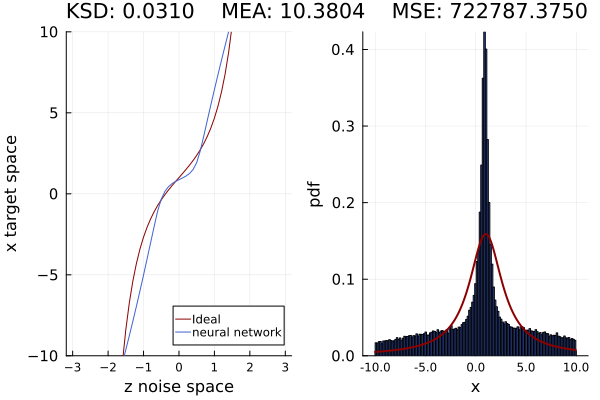}
        \caption*{\tiny Cauchy$(1,2)$, $K_{max}=5$}
    \end{minipage}
    \hfill
    \begin{minipage}{0.22\textwidth}
        \centering
        \includegraphics[width=\linewidth]{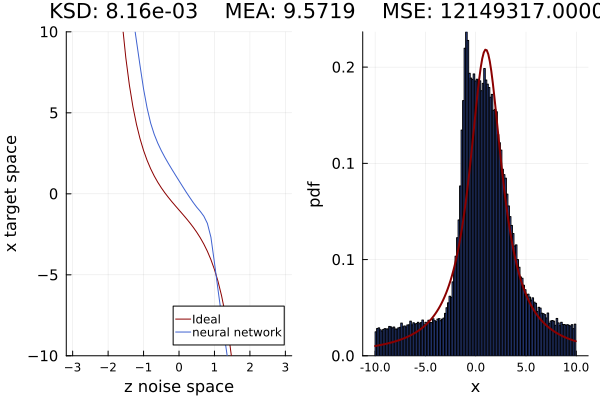}
        \caption*{\tiny Cauchy$(1,2)$, $K_{max}=10$}
    \end{minipage}
    \hfill
    \begin{minipage}{0.22\textwidth}
        \centering
        \includegraphics[width=\linewidth]{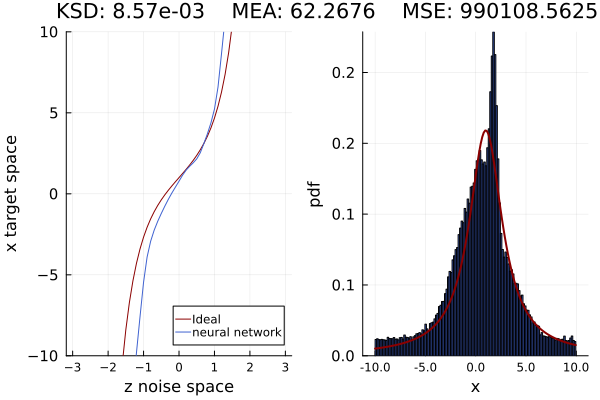}
        \caption*{\tiny Cauchy$(1,2)$, $K_{max}=20$}
    \end{minipage}
    
    \vspace{2cm} 
    
    \begin{minipage}{0.22\textwidth}
        \centering
        \includegraphics[width=\linewidth]{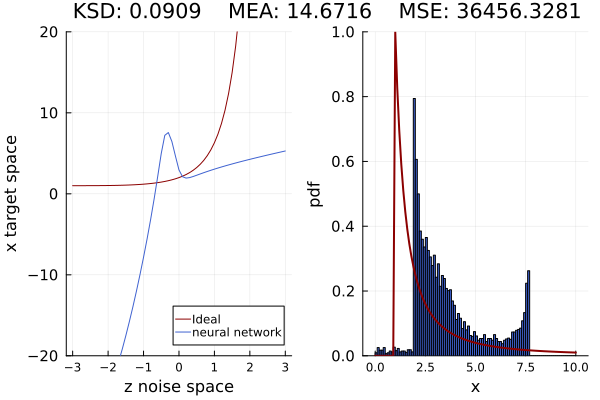}
        \caption*{\tiny Pareto$(1,1)$, $K_{max}=2$}
    \end{minipage}
    \hfill
    \begin{minipage}{0.22\textwidth}
        \centering
        \includegraphics[width=\linewidth]{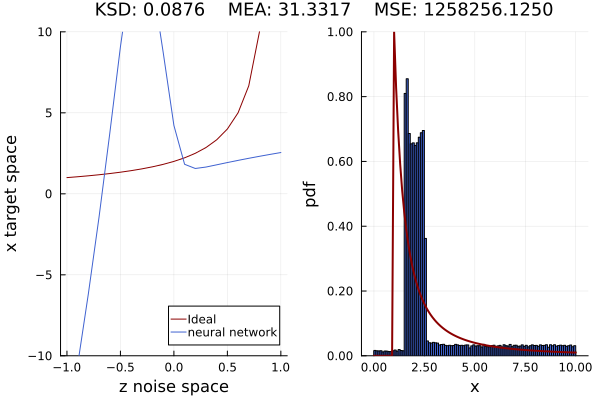}
        \caption*{\tiny Pareto$(1,1)$, $K_{max}=5$}
    \end{minipage}
    \hfill
    \begin{minipage}{0.22\textwidth}
        \centering
        \includegraphics[width=\linewidth]{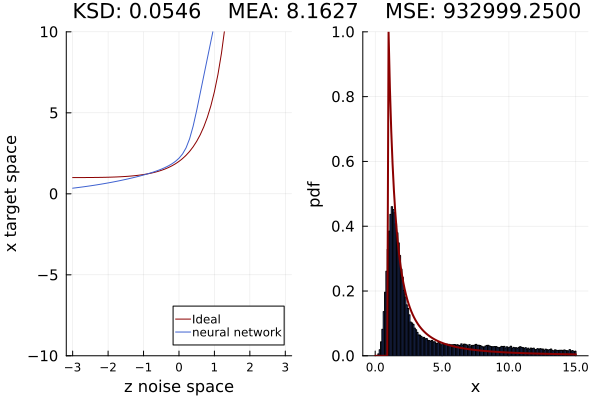}
        \caption*{\tiny Pareto$(1,1)$, $K_{max}=10$}
    \end{minipage}
    \hfill
    \begin{minipage}{0.22\textwidth}
        \centering
        \includegraphics[width=\linewidth]{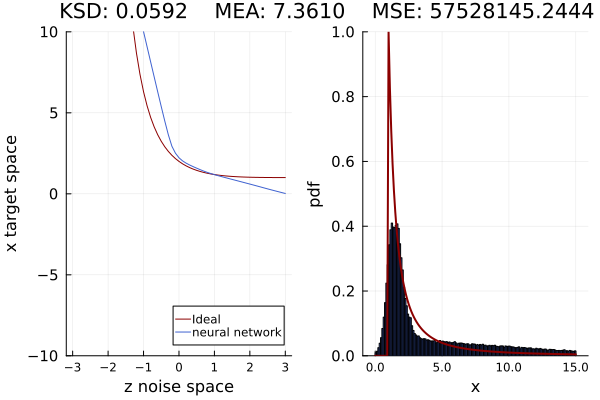}
        \caption*{\tiny Pareto$(1,1)$, $K_{max}=20$}
    \end{minipage}
    
    \caption{Results obtained for varying $K_{max}$ for different target distributions. Global parameters: $N=1000$, $Epochs=1000$, $Learning Rate=10^{-2}$, and Initial Distribution=$\mathcal{N}(0,1)$.}\label{figure 1 appendix}
\end{figure}

\begin{figure}
    \centering
    \begin{minipage}{0.45\textwidth}
        \centering
        \begin{subfigure}{\linewidth}
            \includegraphics[width=\linewidth]{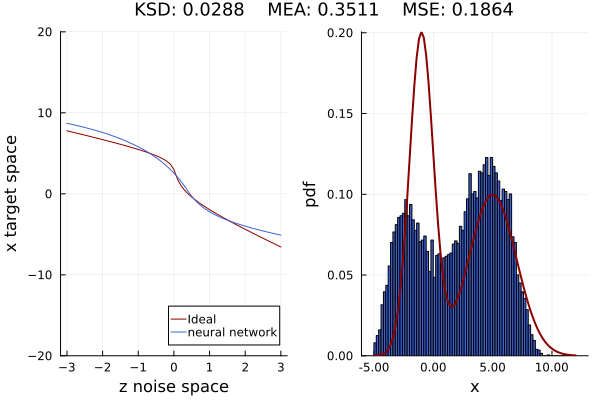}
            \caption*{\tiny $\text{Model}_{1}$, $K_{max}=2$}
        \end{subfigure}
        \par\bigskip\bigskip
        \begin{subfigure}{\linewidth}
            \includegraphics[width=\linewidth]{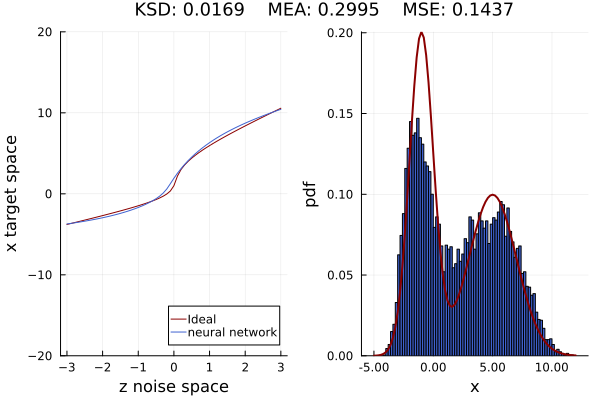}
            \caption*{\tiny $\text{Model}_{1}$, $K_{max}=10$}
        \end{subfigure}
    \end{minipage}
    \hfill
    \begin{minipage}{0.45\textwidth}
        \begin{subfigure}{\linewidth}
            \includegraphics[width=\linewidth]{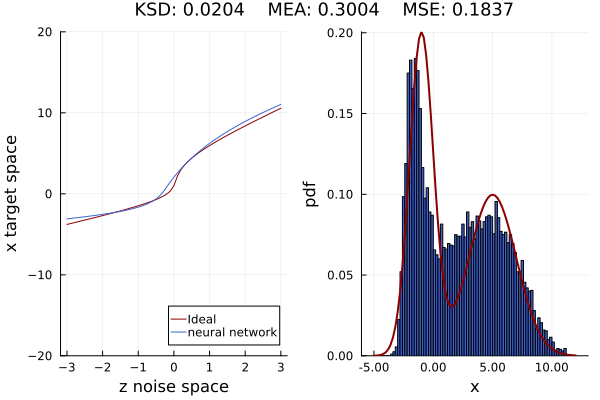}
            \caption*{\tiny $\text{Model}_{1}$, $K_{max}=5$}
        \end{subfigure}
        \par\bigskip\bigskip
        \begin{subfigure}{\linewidth}
            \includegraphics[width=\linewidth]{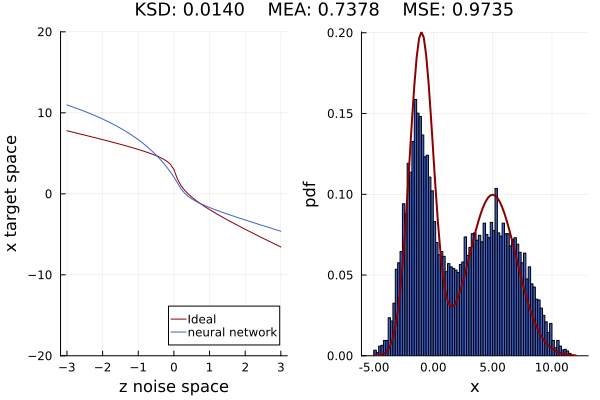}
            \caption*{\tiny $\text{Model}_{1}$, $K_{max}=20$}
        \end{subfigure}
    \end{minipage}

    \caption{Results obtained for varying $K_{max}$ for Target Distributions = $\text{Model}_{1}$.\\Global parameters: $N=1000$, $Epochs=1000$, $Learning Rate=10^{-2}$, and Initial Distribution=$\normdist{0}{1}$.} \label{figure 2 appendix}
\end{figure}

\newpage

As before, we are analyzing the effects of the evolution of the hyperparameter $K_{max}$ for different distributions. In this case, we consider the random noise originating from an uniform distribution $\mathcal{U}(-1,1)$. The other settings are identical to the $\normdist{0}{1}$ case.
\vspace{1cm}


\begin{figure}[H]
    \centering
    
    \begin{minipage}{0.22\textwidth}
        \centering
        \includegraphics[width=\linewidth]{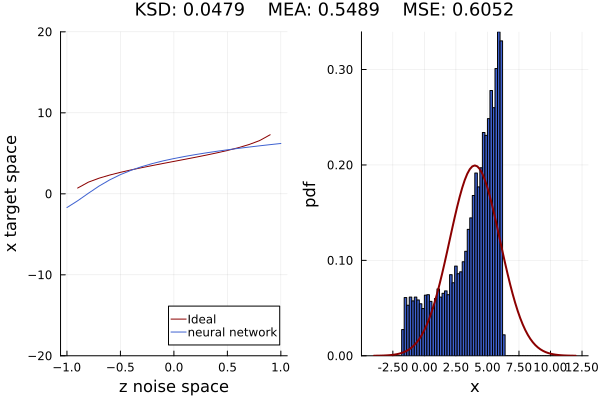}
        \caption*{\tiny $\mathcal{N}(4,2)$, $K_{max}=2$}
    \end{minipage}
    \hfill
    \begin{minipage}{0.22\textwidth}
        \centering
        \includegraphics[width=\linewidth]{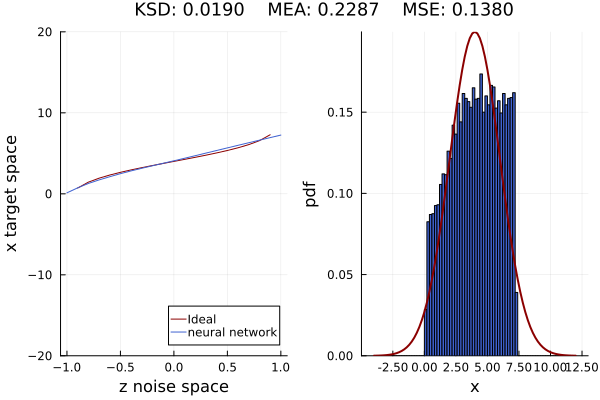}
        \caption*{\tiny $\mathcal{N}(4,2)$, $K_{max}=5$}
    \end{minipage}
    \hfill
    \begin{minipage}{0.22\textwidth}
        \centering
        \includegraphics[width=\linewidth]{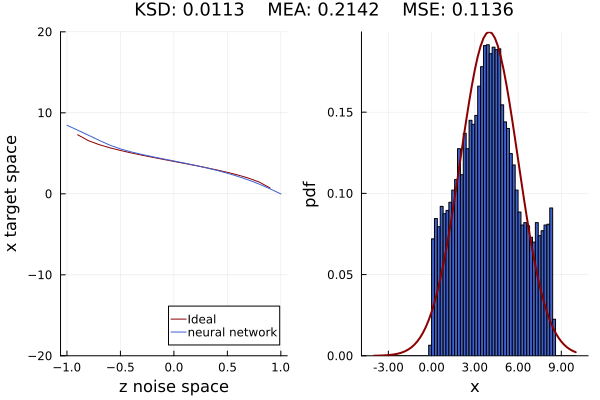}
        \caption*{\tiny $\mathcal{N}(4,2)$, $K_{max}=10$}
    \end{minipage}
    \hfill
    \begin{minipage}{0.22\textwidth}
        \centering
        \includegraphics[width=\linewidth]{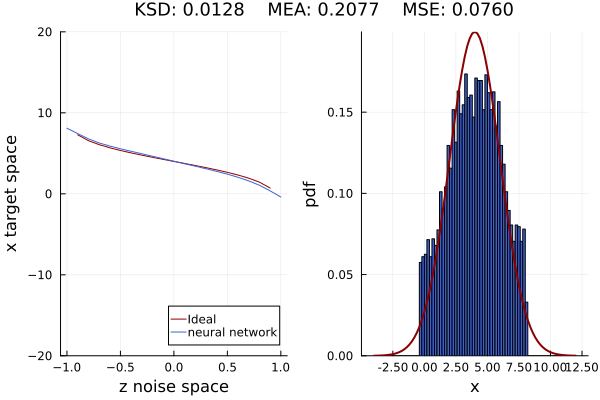}
        \caption*{\tiny $\mathcal{N}(4,2)$, $K_{max}=20$}
    \end{minipage}

    \vspace{2cm} 

    \begin{minipage}{0.22\textwidth}
        \centering
        \includegraphics[width=\linewidth]{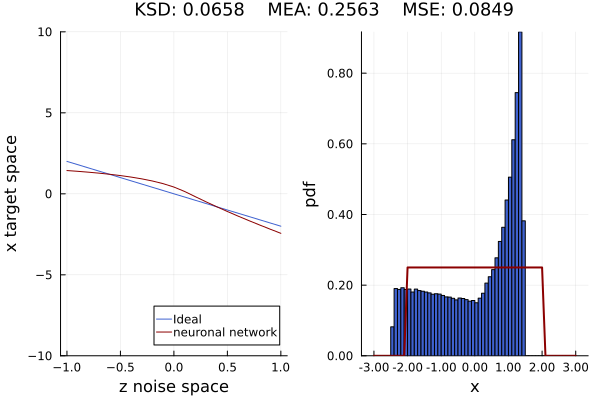}
        \caption*{\tiny $\mathcal{U}(4, 2)$, $K_{max}=2$}
    \end{minipage}
    \hfill
    \begin{minipage}{0.22\textwidth}
        \centering
        \includegraphics[width=\linewidth]{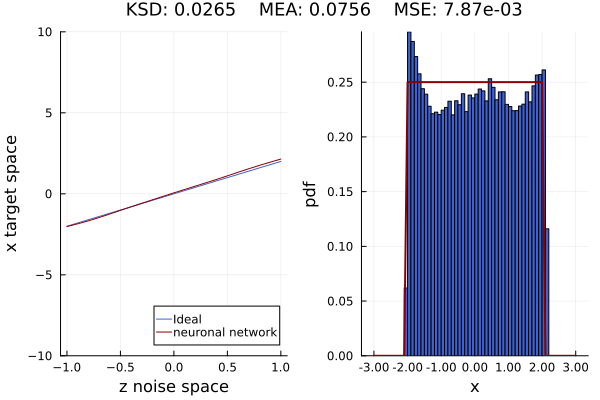}
        \caption*{\tiny $\mathcal{U}(4, 2)$, $K_{max}=5$}
    \end{minipage}
    \hfill
    \begin{minipage}{0.22\textwidth}
        \centering
        \includegraphics[width=\linewidth]{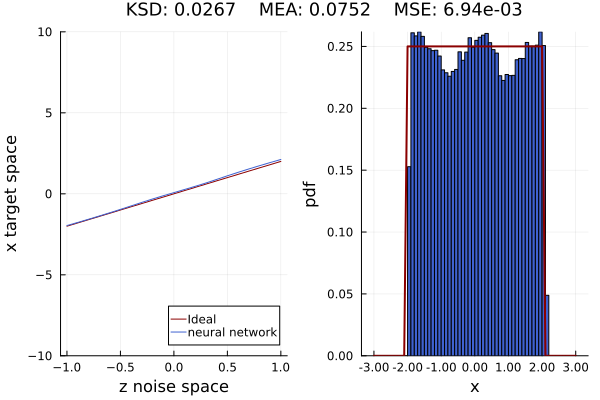}
        \caption*{\tiny $\mathcal{U}(4, 2)$, $K_{max}=10$}
    \end{minipage}
    \hfill
    \begin{minipage}{0.22\textwidth}
        \centering
        \includegraphics[width=\linewidth]{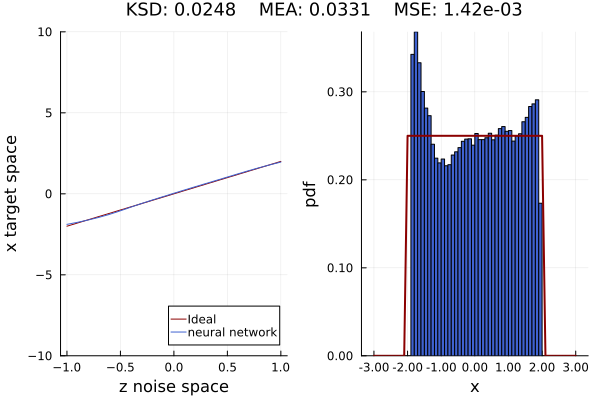}
        \caption*{\tiny $\mathcal{U}(4, 2)$, $K_{max}=20$}
    \end{minipage}

    \vspace{2cm} 

    \begin{minipage}{0.22\textwidth}
        \centering
        \includegraphics[width=\linewidth]{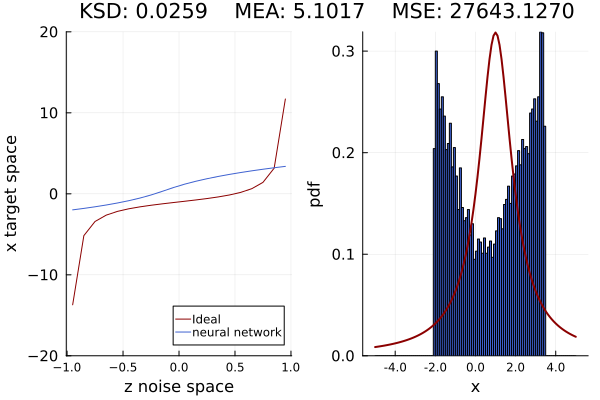}
        \caption*{\tiny Cauchy$(1,2)$, $K_{max}=2$}
    \end{minipage}
    \hfill
    \begin{minipage}{0.22\textwidth}
        \centering
        \includegraphics[width=\linewidth]{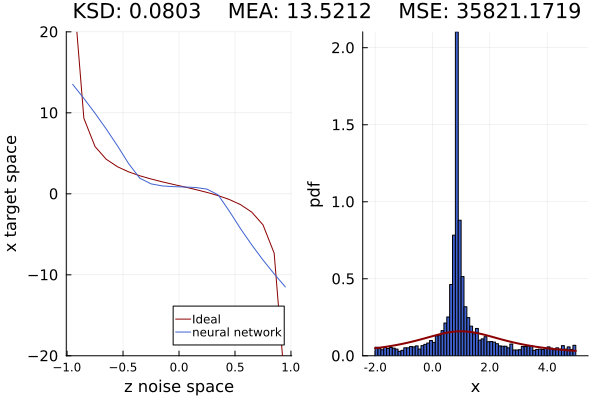}
        \caption*{\tiny Cauchy$(1,2)$, $K_{max}=5$}
    \end{minipage}
    \hfill
    \begin{minipage}{0.22\textwidth}
        \centering
        \includegraphics[width=\linewidth]{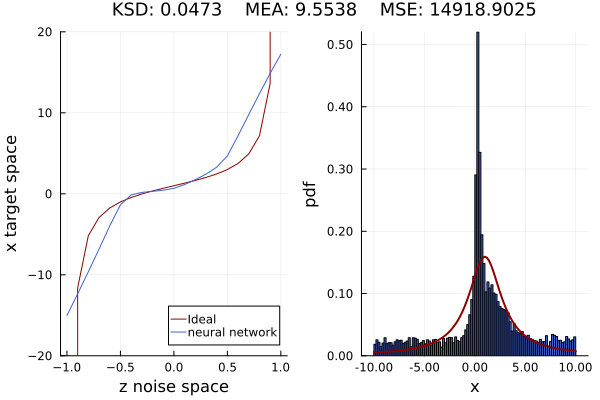}
        \caption*{\tiny Cauchy$(1,2)$, $K_{max}=10$}
    \end{minipage}
    \hfill
    \begin{minipage}{0.22\textwidth}
        \centering
        \includegraphics[width=\linewidth]{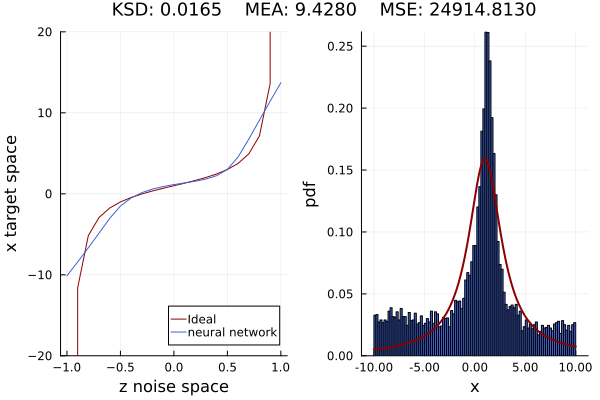}
        \caption*{\tiny Cauchy$(1,2)$, $K_{max}=20$}
    \end{minipage}

    \vspace{2cm} 

    \begin{minipage}{0.22\textwidth}
        \centering
        \includegraphics[width=\linewidth]{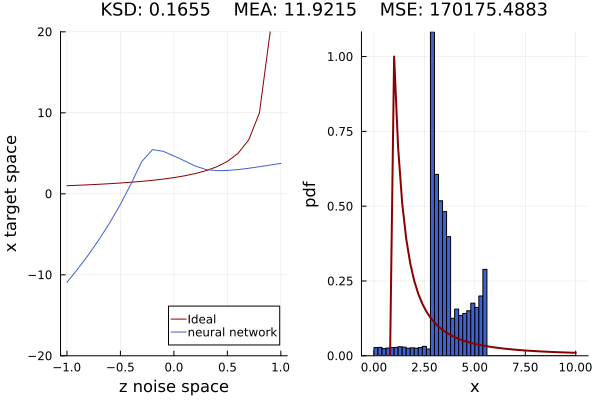}
        \caption*{\tiny Pareto$(1,1)$, $K_{max}=2$}
    \end{minipage}
    \hfill
    \begin{minipage}{0.22\textwidth}
        \centering
        \includegraphics[width=\linewidth]{img/Learning1D/Unif/Pareto/samples_1000_max_k_5_epochs_1000_lr_001_1.png}
        \caption*{\tiny Pareto$(1,1)$, $K_{max}=5$}
    \end{minipage}
    \hfill
    \begin{minipage}{0.22\textwidth}
        \centering
        \includegraphics[width=\linewidth]{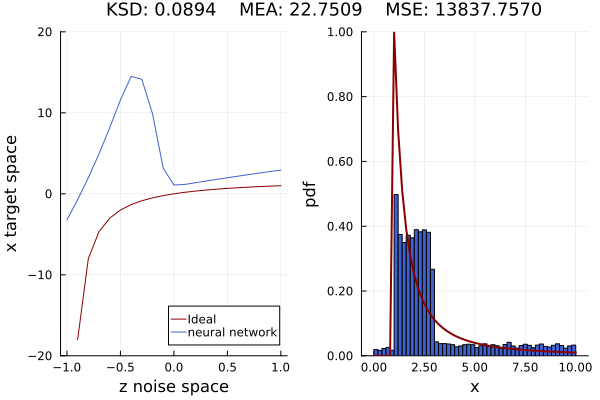}
        \caption*{\tiny Pareto$(1,1)$, $K_{max}=10$}
    \end{minipage}
    \hfill
    \begin{minipage}{0.22\textwidth}
        \centering
        \includegraphics[width=\linewidth]{img/Learning1D/Unif/Pareto/samples_1000_max_k_10_epochs_1000_lr_001_2_1.png}
        \caption*{\tiny Pareto$(1,1)$, $K_{max}=20$}
    \end{minipage}

    \caption{Results obtained for varying $K_{max}$ for different target distributions. Global parameters: $N=1000$, $Epochs=1000$, $\text{Learning Rate}=10^{-2}$, and Initial Distribution = $\mathcal{U}(-1,1)$.}
\end{figure}

\clearpage
\newpage

\begin{figure}
    \centering
    \begin{minipage}{0.45\textwidth}
        \centering
        \begin{subfigure}{\linewidth}
            \includegraphics[width=\linewidth]{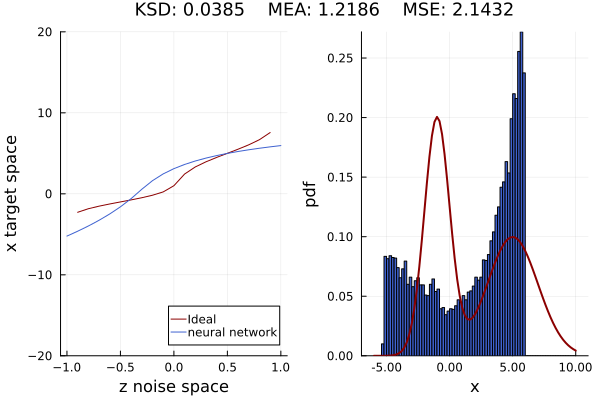}
            \caption*{\tiny $\text{Model}_{1}$, $K_{max}=2$}
        \end{subfigure}
        \par\bigskip\bigskip
        \begin{subfigure}{\linewidth}
            \includegraphics[width=\linewidth]{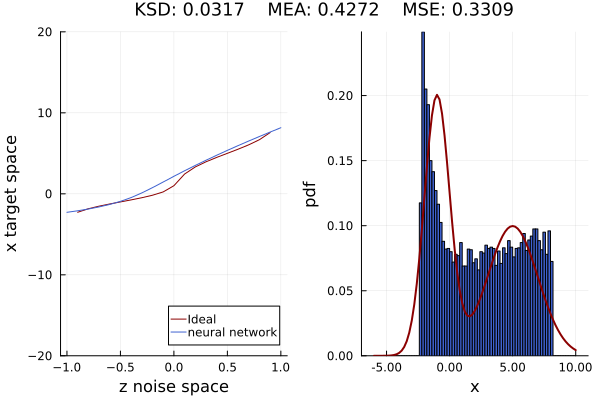}
            \caption*{\tiny $\text{Model}_{1}$, $K_{max}=5$}
        \end{subfigure}
    \end{minipage}
    \begin{minipage}{0.45\textwidth}
        \begin{subfigure}{\linewidth}
            \includegraphics[width=\linewidth]{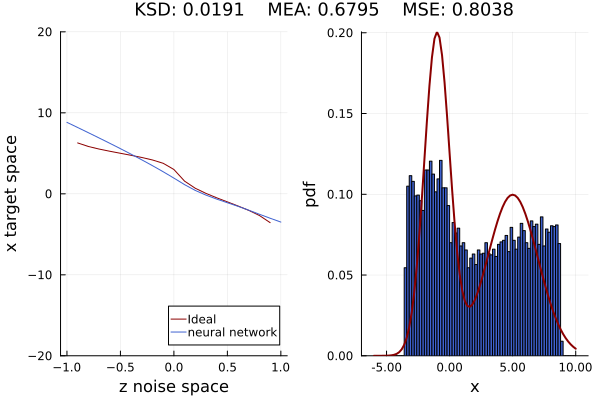}
            \caption*{\tiny $\text{Model}_{1}$, $K_{max}=20$} 
        \end{subfigure}
        \par\bigskip\bigskip
        \begin{subfigure}{\linewidth}
            \includegraphics[width=\linewidth]{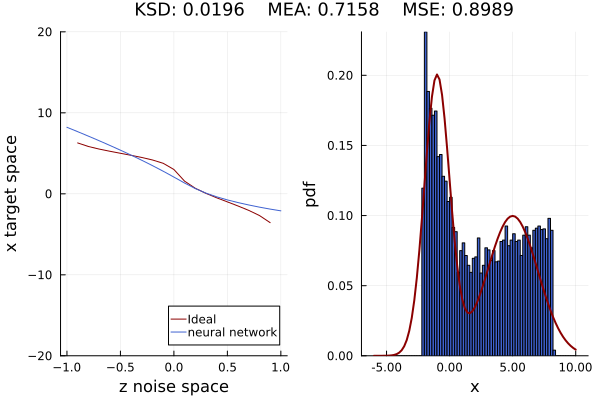}
            \caption*{\tiny $\text{Model}_{1}$, $K_{max}=10$} 
        \end{subfigure}
    \end{minipage}
    \caption{Results obtained for varying $K_{max}$ for Target Distributions = $\text{Model}_{1}$. Global parameters: $N=1000$, $Epochs=1000$, $\text{Learning Rate}=10^{-2}$, and Initial Distribution = $\mathcal{U}(0, 1)$.}
\end{figure}

\newpage
\clearpage

\subsection{Experiments with large hyperparameters $K$, $N$ and number of epochs}
In this section, we investigate the capacity of ISL to learn complex distributions in scenarios where the hyperparameters $K$ and $N$ are set to large values. In each case, we specify the values of the hyperparameters. As we can see, in these two examples with large values of $K_{max}$, $N$, and allowing for sufficient training time (number of epochs is large), the method is capable of learning complex multimodal distributions.


\begin{figure}[h]
   \begin{minipage}{0.45\textwidth}
        \centering
        \begin{subfigure}{\linewidth}
            \includegraphics[width=\linewidth]{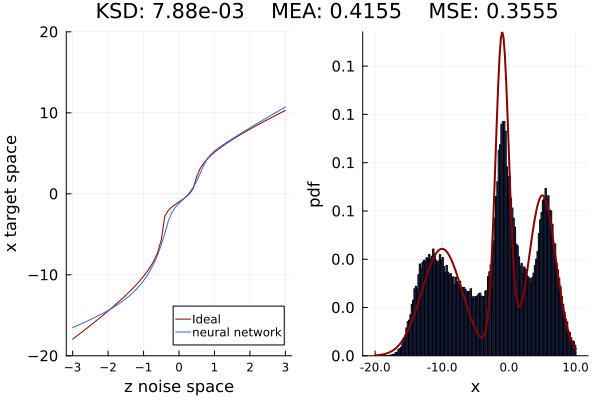}
            \caption{\small $N=10000$, $K_{max}=10$, epochs=$10000$, $\text{lr} = 10^{-2}$,\\ $\text{Target Distribution = \text{Model}}_3$.}
        \end{subfigure}
    \end{minipage}
    \hfill
    \begin{minipage}{0.45\textwidth}
    \centering
        \begin{subfigure}{\linewidth}
            \includegraphics[width=\linewidth]{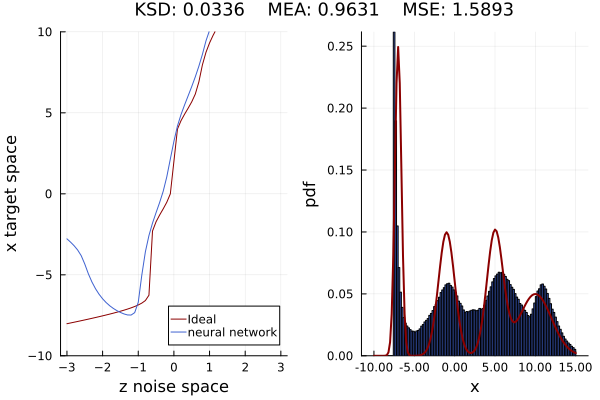}
            \caption{\small $N=5000$, $K_{max}=50$, epochs=$10000$, $\text{lr}=10^{-2}$\\
            $\text{Target Distribution =} \\ \text{MixtureModel}\big ( \normdist{-7}{0.5}, \normdist{-1}{1}, \normdist{4}{1},  \normdist{10}{2}\big)$}
        \end{subfigure}
    \end{minipage}
    \caption{Complex Distribution Learning with High Values of Hyperparameters: $K$, $N$, and Epochs.}

\end{figure}
\hfill
\subsection{Evolution of the hyperparameter $K$}

Here, we study the evolution of
hyperparameter $K$ during training. It is progressively adapted according to the algorithm explained in Section 4.2 of the paper.
In every experiment we use a learning rate of $10^{-2}$ and train for 10000 epochs. The number of ground-truth observations is $N=2,000$. In this case, we haven't imposed any restrictions on $K_{max}$, and hence $K$ can grow as required. In the subcaption, the target distribution is specified. The initial distribution is a $\normdist{0}{1}$ in all cases.

We notice that initially, as the number of epochs increases, the value of $K$ grows very rapidly, but it eventually levels off after reaching a certain point. 
The observed data (see Figure \ref{figure 6}) leads us to infer that obtaining a high $K$ value requires a progressively more effort as compared to attaining a lower value, as evidenced by the incremental increases over successive training epochs. This complexity seems to exhibit a logarithmic pattern. Therefore, there is an advantage in initially setting smaller values of $K$ and gradually increasing them.

\begin{figure}[p]
   \begin{minipage}{0.42\textwidth}
        \centering
        \begin{subfigure}{\linewidth}
            \includegraphics[width=\linewidth]{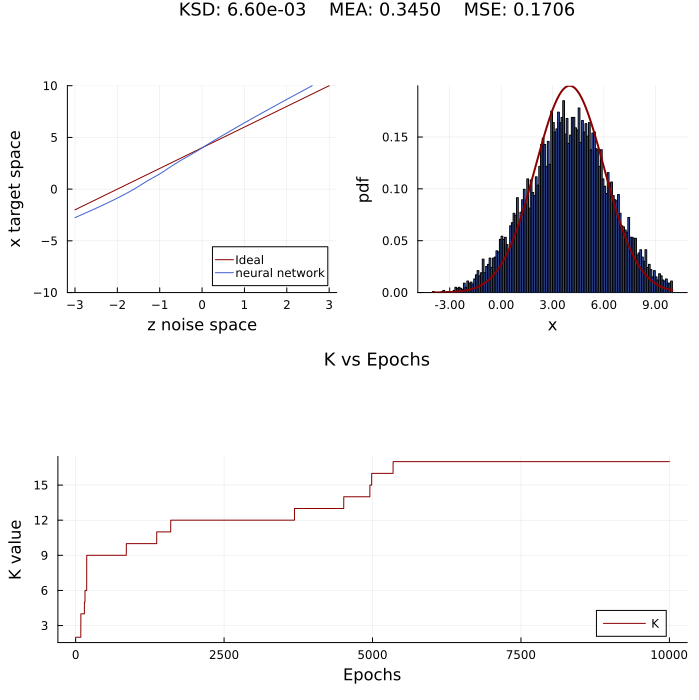}
            \caption{\small $\normdist{4}{2}$}
        \end{subfigure}
        \par\bigskip\bigskip
        \begin{subfigure}{\linewidth}
            \includegraphics[width=\linewidth]{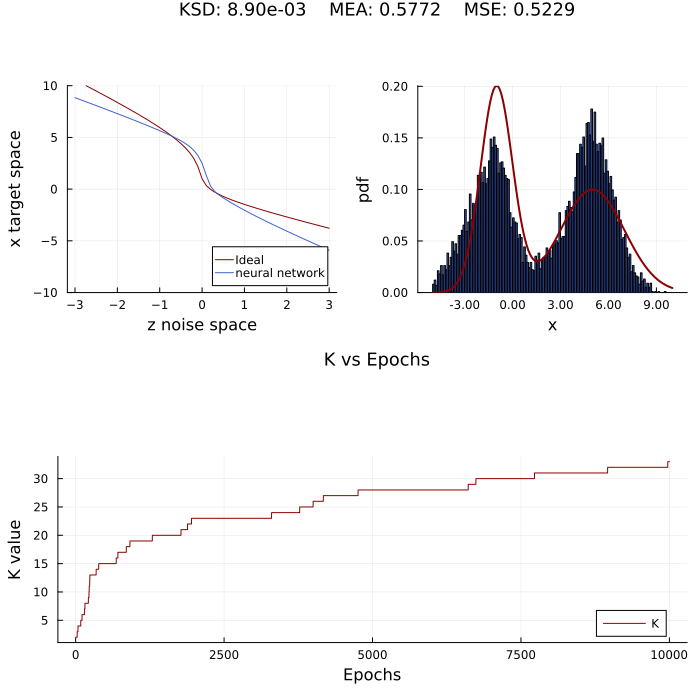}
            \caption{\small $\text{Model}_{1}$}
        \end{subfigure}
    \end{minipage}
    \hfill
    \begin{minipage}{0.42\textwidth}
        \begin{subfigure}{\linewidth}
            \includegraphics[width=\linewidth]{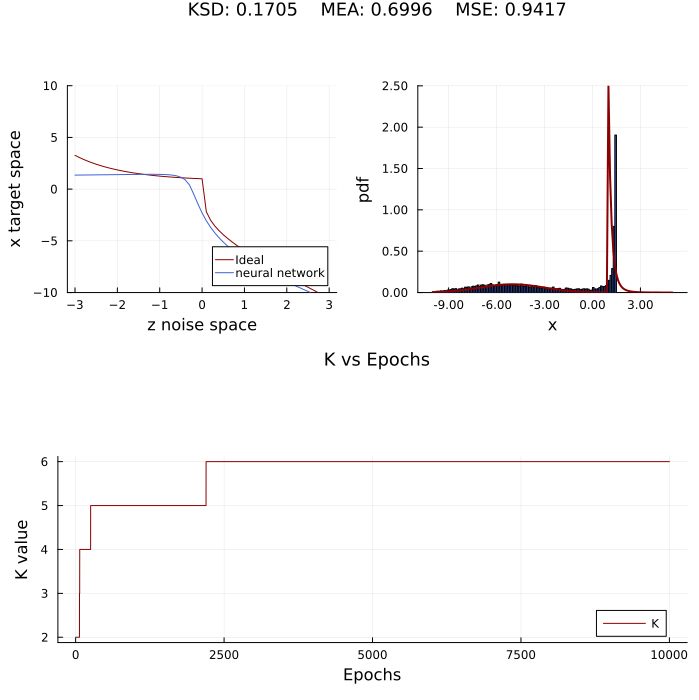}
            \caption{\small $\text{Model}_{3}$}
        \end{subfigure}
        \par\bigskip\bigskip
        \begin{subfigure}{\linewidth}
            \includegraphics[width=\linewidth]{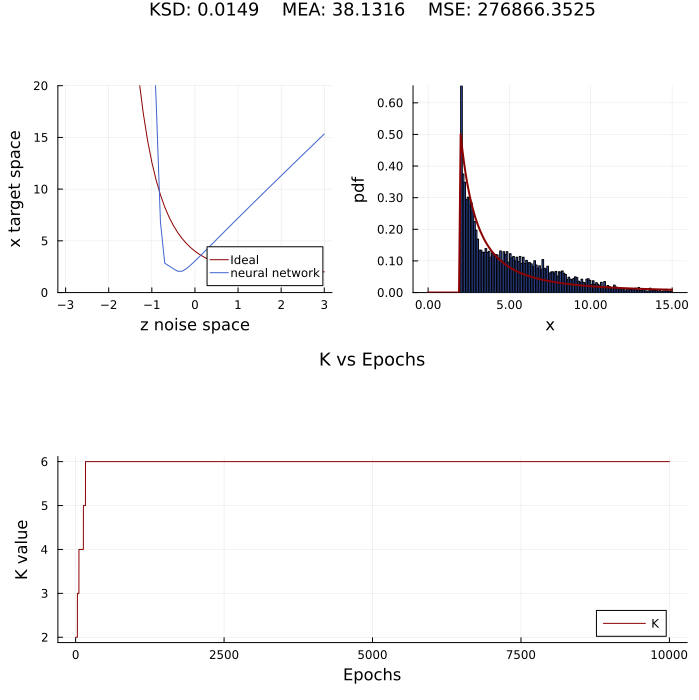}
            \caption{\small Pareto$(1,2)$, $N = 2000$}
        \end{subfigure}
    \end{minipage}
    \caption{\small Evolution of the hyperparameter $K$ as a function of the epochs. Figure top left: We compare the optimal transformation from a $\normdist{0}{1}$ to the target distribution versus the generator learned by ISL, Figure top right: The resulting distribution, Figure at the bottom: The evolution of $K$ as a function of epochs.} \label{figure 6}
\end{figure}

\clearpage
\newpage

\subsection{Mixture time series}



We now consider the following composite time series, such that the underlying density function is multimodal. For each time \( t \), a sample is taken from functions,
\[
\begin{array}{ccc}
    y_1(t) = 10 \cos(t-0.5) + \xi
    &
    \text{or,}
    &
    y_2(t) = 10 \cos(t-0.5) + \xi,
\end{array}
\]
based on a Bernoulli distribution with parameter \( 1/2 \). Noise, denoted by \( \xi \), follows a standard Gaussian distribution, $\normdist{0}{1}$. Our temporal variable, \( t \),  will take values in the interval $(-4,4)$ where each point will be equidistantly spaced by $0.1$ units apart (when we represent it, we use the conventional notation $t\geq 0$). The neural network comprises two components: an RNN with three layers featuring 16, 16 and finally 32 units; and a feedforward neural network with three layers  containing 32, 64 units and 64 units. Activation functions include ELU for the RNN layers and ELU for the MLP layers, except for the final layer, which employs an identity activation function. The chosen learning rate is $10^{-3}$ in an Adam optimizer.

In the Figure \ref{Mixture time series}, we display the predictions obtained for the neural network trained using ISL and DeepAR, with a prediction window of 20 steps.
\vspace{1cm}

\begin{figure}[H]
    \centerline{
   \begin{subfigure}{0.5\textwidth}
        \centering
        \includegraphics[width=\textwidth]{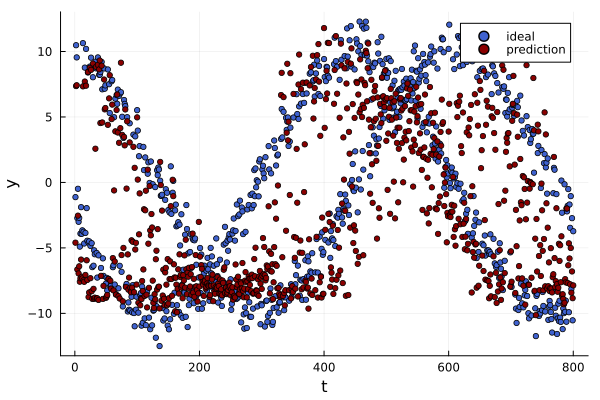}
        \caption{\tiny ISL forecasting}
        \label{fancy figure:sub_a}
    \end{subfigure}
    \hfill
    \begin{subfigure}{0.5\textwidth}
        \centering
        \includegraphics[width=\textwidth]{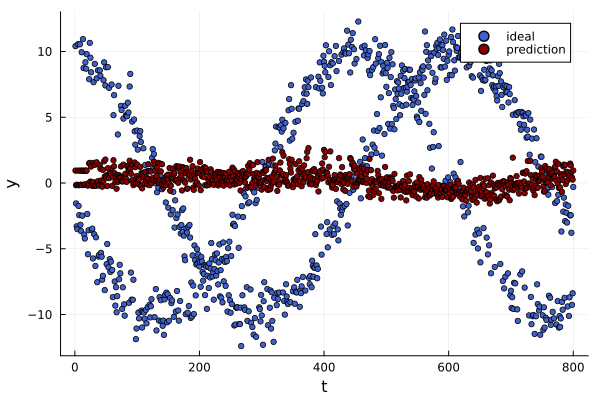}
        \caption{\tiny DeepAR forecasting}
        \label{fancy figure:sub_b}
    \end{subfigure}
    }
    \caption{Mixture time series, forecasting 20-step}
    \label{Mixture time series}
\end{figure}

\clearpage
\newpage

\subsection{More experiments with real world datasets}
In this section, we delve deeper into the experiments carried out with real world datasets. Previously, we provided a brief overview of the results pertaining to `electric-f' and `electric-c.' Within this section, we will present comparative graphs for these two datasets and introduce the results related to the wind dataset.

\subsubsection{Electricity-f time series}
We present comparative graphs of the results achieved by ISL using the `electricity-f' dataset. The training phase utilized data from 2012 to 2013, while validation encompassed a span of 20 months, commencing on April 15, 2014. For this purpose, we employed a model architecture consisting of 2 layers-deep RNN with 3 units each, as well as a 2-layer MLP with 10 units each and a ReLU activation function.

In the case of training with DeepAR, we utilized the GluonTS library \cite{alexandrov2020gluonts} with its default settings, which entail $2$ recurrent layers, each comprising $40$ units. The model was trained over the course of $100$ epochs. The results presented in Figure \ref{1-day prediction of `electricity-c'} relate to the prediction of the final day within the series for different users, with the shaded area indicating the range of plus or minus one standard deviation from the predicted value.
\vspace{1cm}
\begin{figure}[H]
   \begin{minipage}{0.45\textwidth}
        \centering
        \begin{subfigure}{\linewidth}
            \includegraphics[width=\linewidth]{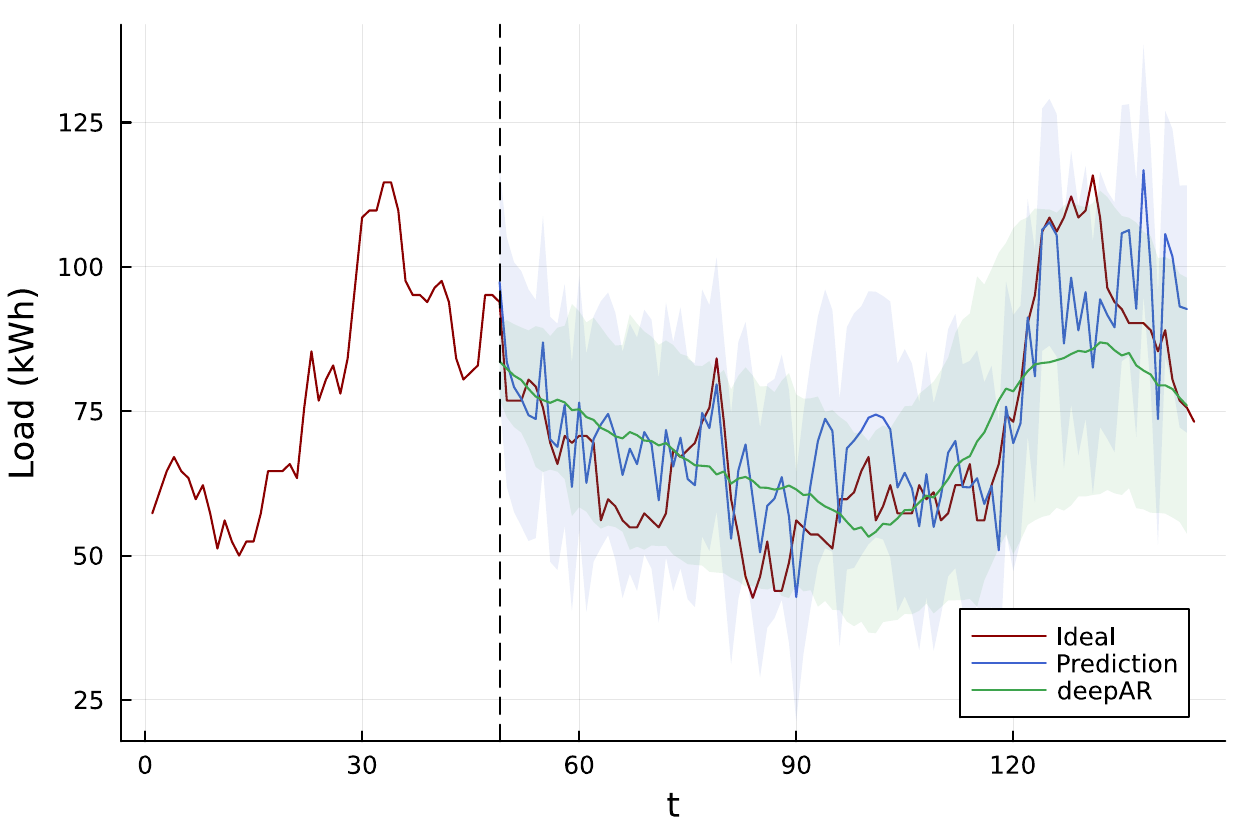}
            \caption{\small user 5}
        \end{subfigure}
        \par\bigskip\bigskip
        \begin{subfigure}{\linewidth}
            \includegraphics[width=\linewidth]{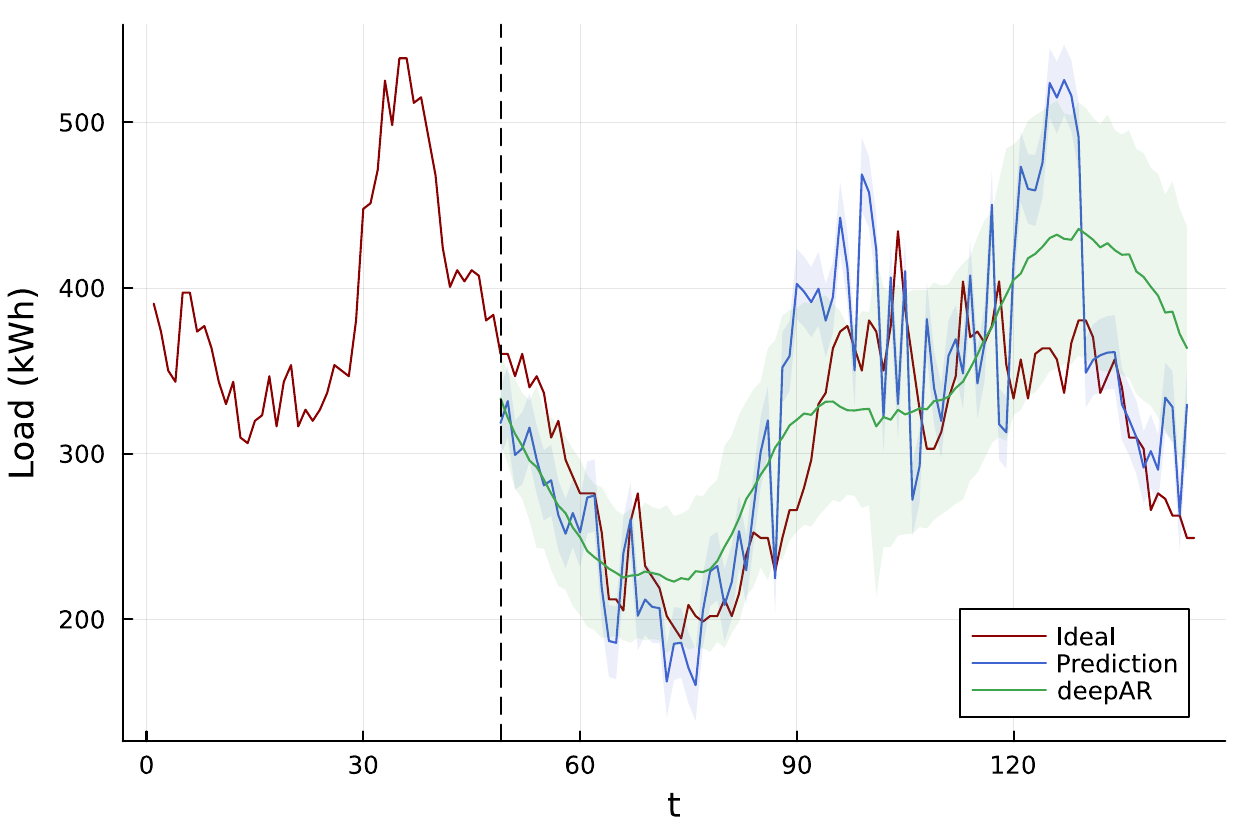}
            \caption{\small user 8}
        \end{subfigure}
    \end{minipage}
    \hfill
    \begin{minipage}{0.45\textwidth}
        \begin{subfigure}{\linewidth}
            \includegraphics[width=\linewidth]{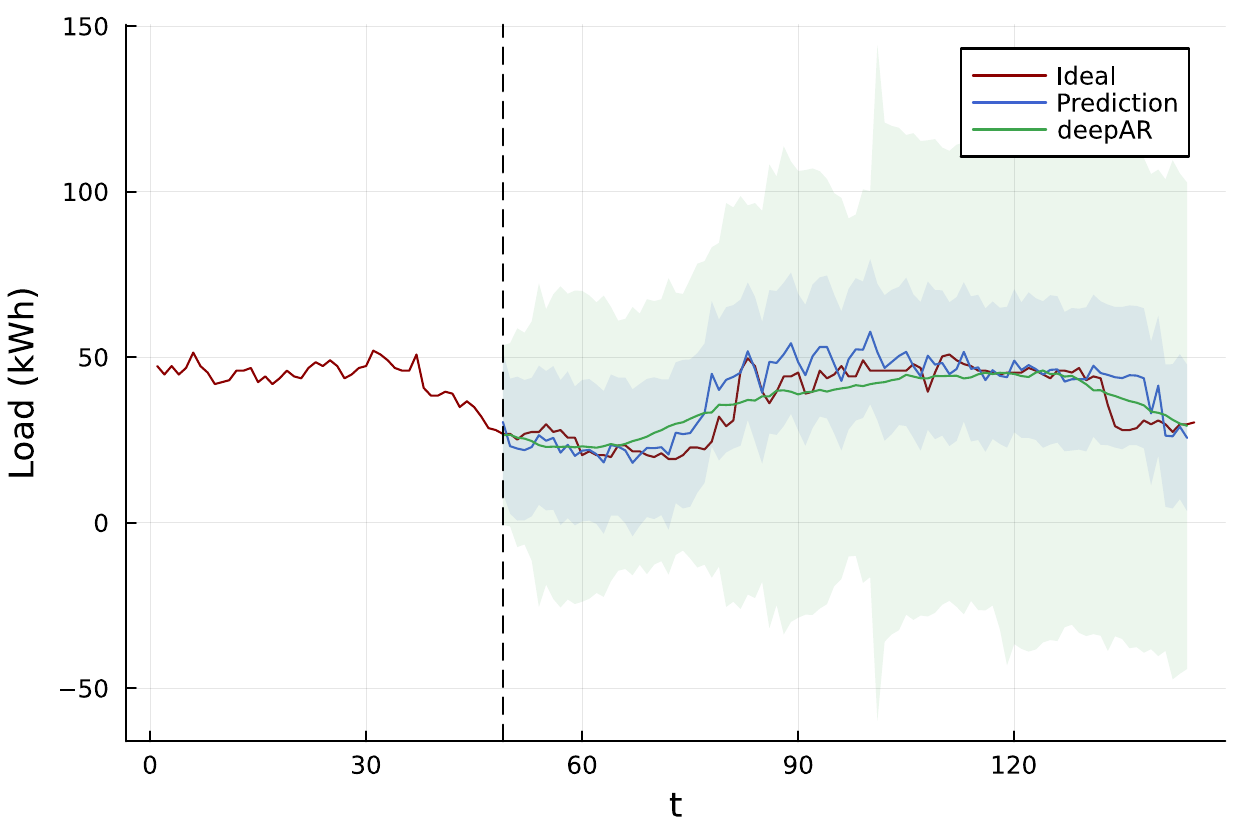}
            \caption{\small user 169}
        \end{subfigure}
        \par\bigskip\bigskip
        \begin{subfigure}{\linewidth}
            \includegraphics[width=\linewidth]{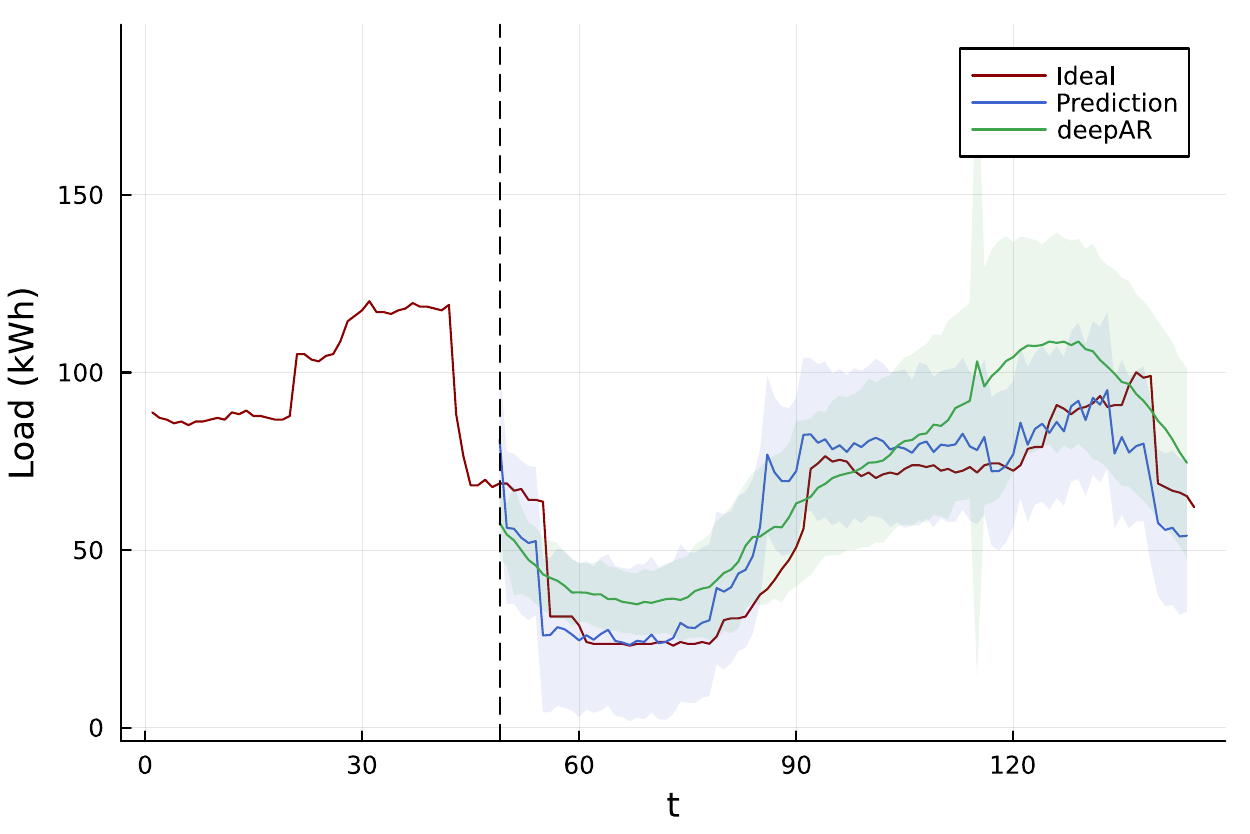}
            \caption{\small user 335}
        \end{subfigure}
    \end{minipage}
    \caption{\small 1-day prediction of `electricity-f' consumption. Blue: 1-day prediction using ISL, Green: 1-day prediction using DeepAR.} \label{1-day prediction of `electricity-c'}
\end{figure}
\hfill

\clearpage
\newpage
\subsubsection{Electricity-c time series}

We present comparative graphs of the results achieved by ISL using the `electricity-c' dataset. We used the same settings for both training and validation, as well as the same neural network as described for the `electricity-f' dataset. Below, we present various prediction plots for different users, both for a $1$-day (Figure \ref{1-day forecast}) and a $7$-day forecast (Figure \ref{7-day forecast}).

\begin{figure}[H]\label{1-day forecast}
   \begin{minipage}{0.5\textwidth}
        \centering
        \begin{subfigure}{\linewidth}
            \includegraphics[width=\linewidth]{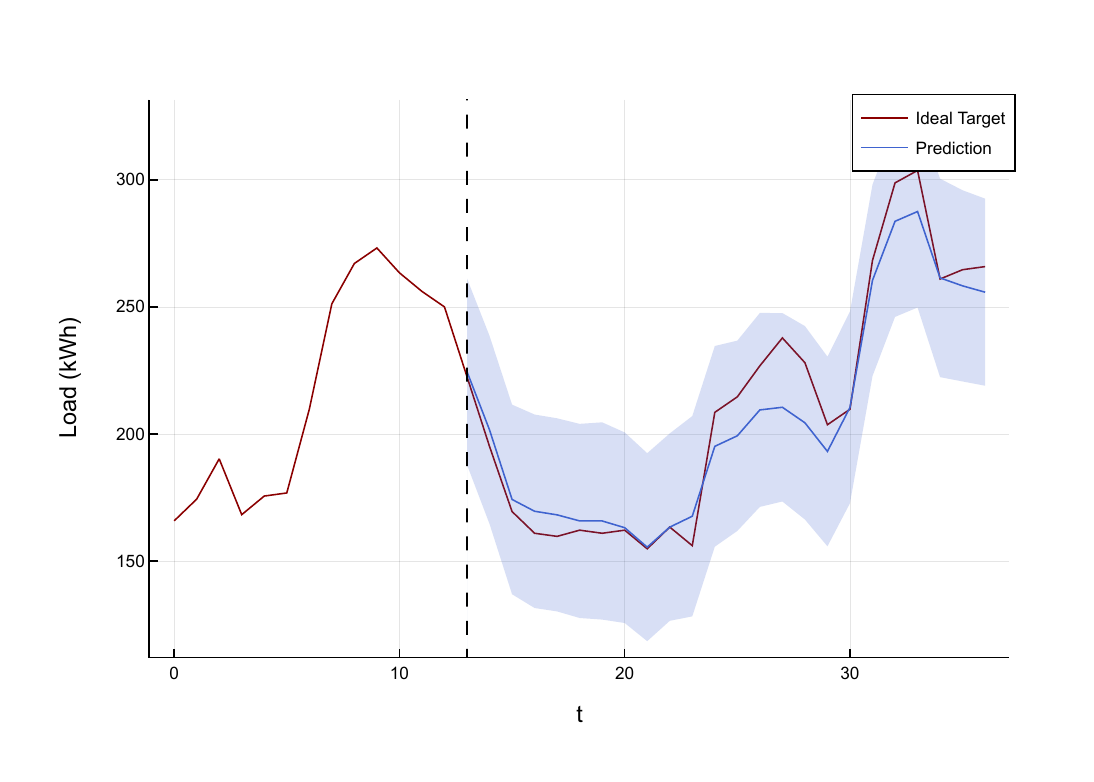}
            \caption{\small user 5}
        \end{subfigure}
        \par\bigskip\bigskip
        \begin{subfigure}{\linewidth}
            \includegraphics[width=\linewidth]{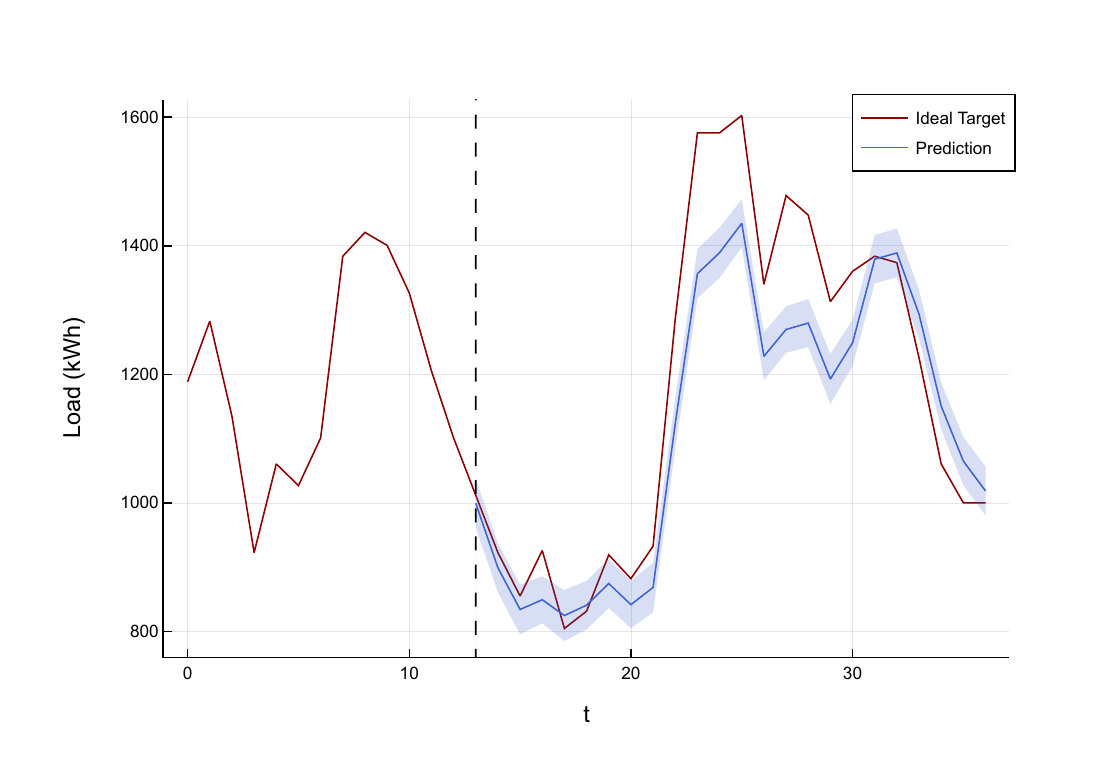}
            \caption{\small user 8}
        \end{subfigure}
    \end{minipage}
    \hfill
    \begin{minipage}{0.5\textwidth}
        \begin{subfigure}{\linewidth}
            \includegraphics[width=\linewidth]{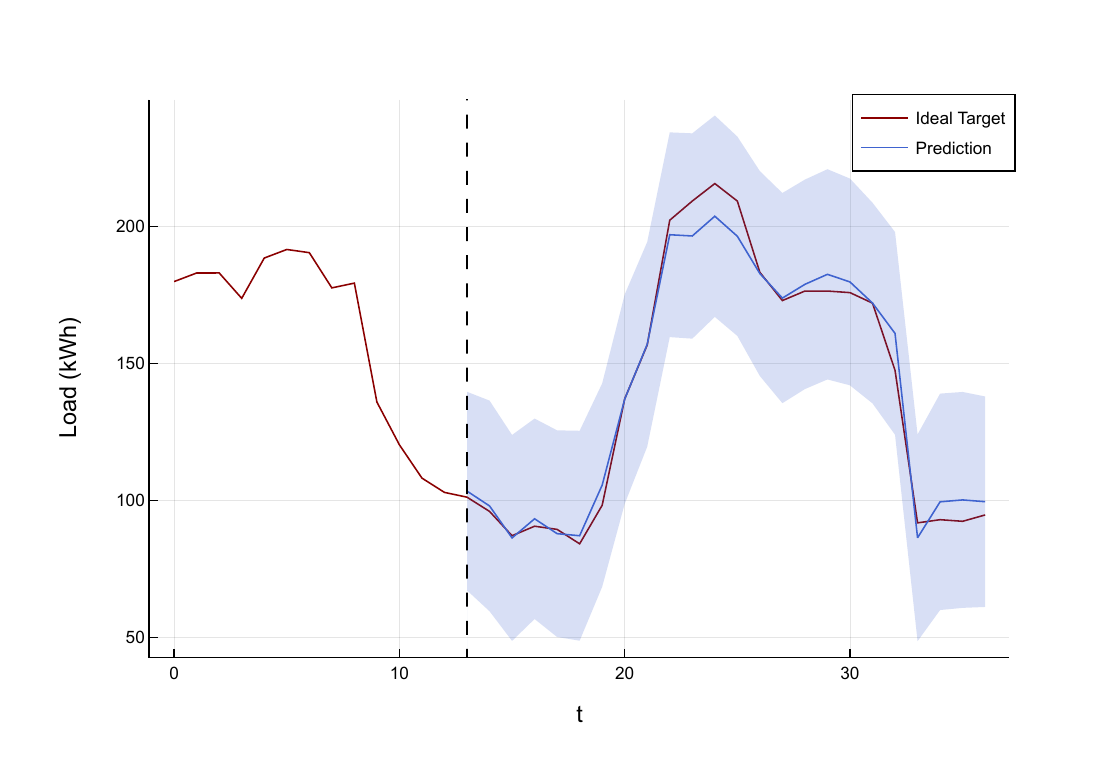}
            \caption{\small user 169}
        \end{subfigure}
        \par\bigskip\bigskip
        \begin{subfigure}{\linewidth}
            \includegraphics[width=\linewidth]{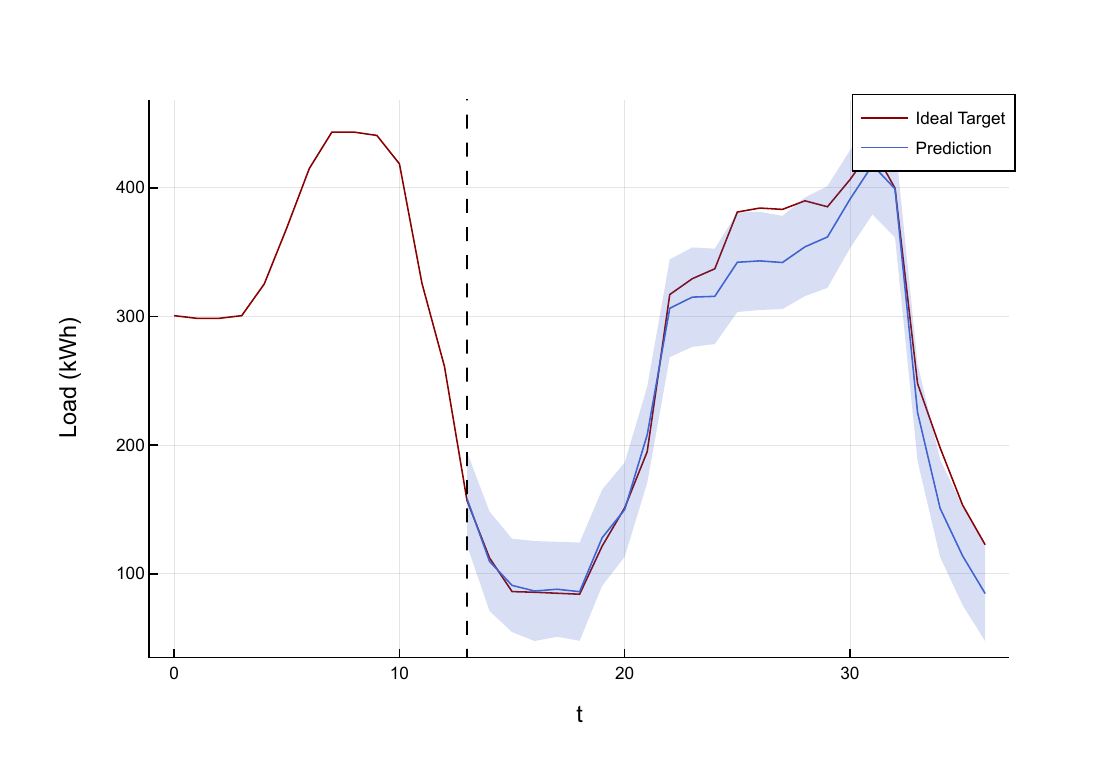}
            \caption{\small user 335}
        \end{subfigure}
    \end{minipage}
    \caption{1-day forecast for `electricity-c' consumption}\label{1-day forecast}
\end{figure}
\hfill

\begin{figure}[H] \label{7-day forecast}
   \begin{minipage}{0.5\textwidth}
        \centering
        \begin{subfigure}{\linewidth}
            \includegraphics[width=\linewidth]{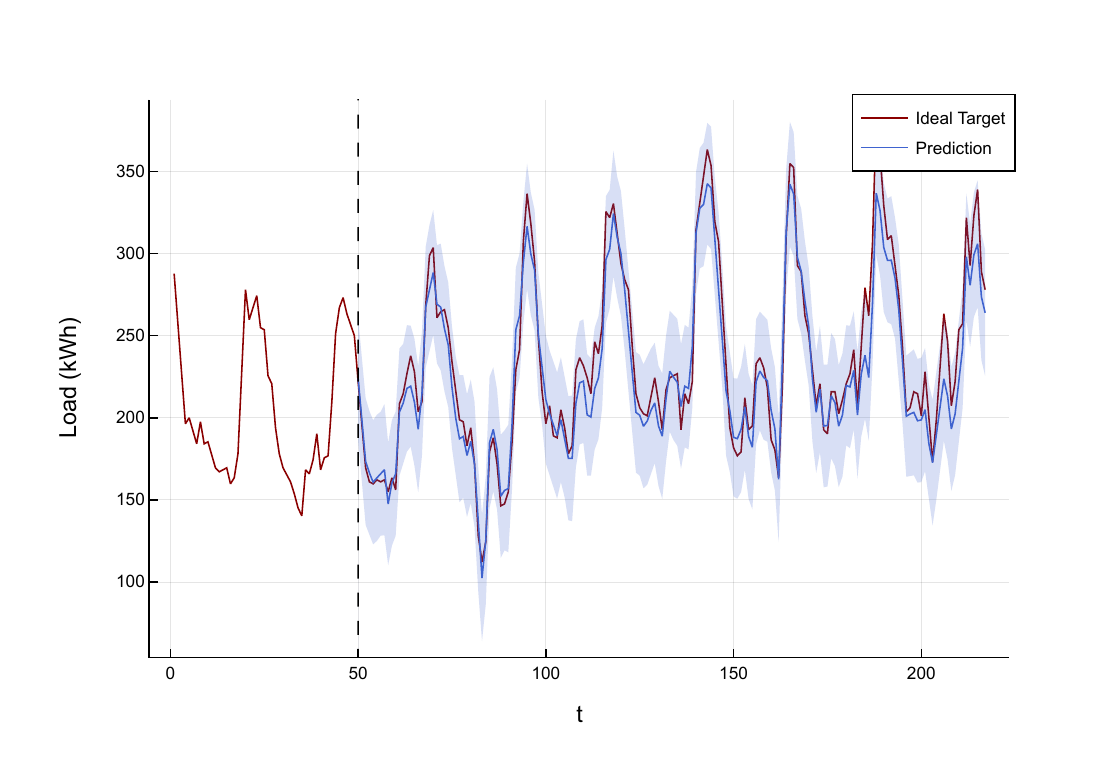}
            \caption{\small user 5}
        \end{subfigure}
        \par\bigskip\bigskip
        \begin{subfigure}{\linewidth}
            \includegraphics[width=\linewidth]{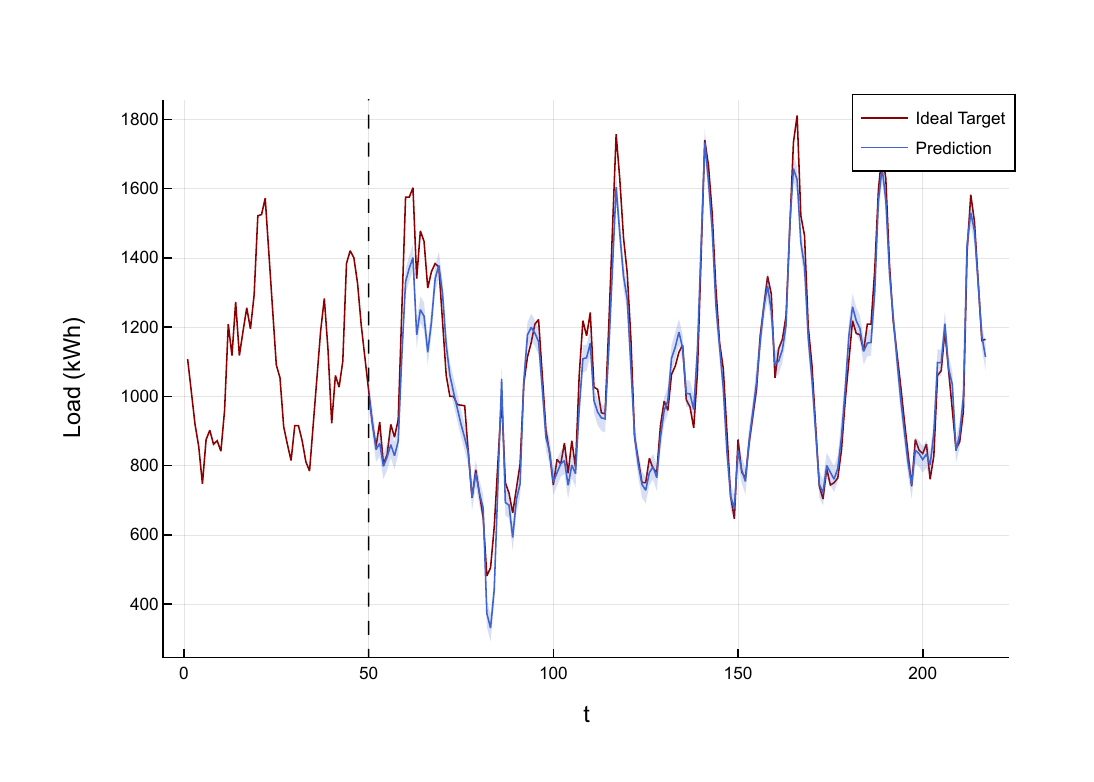}
            \caption{\small user 8}
        \end{subfigure}
    \end{minipage}
    \hfill
    \begin{minipage}{0.5\textwidth}
        \begin{subfigure}{\linewidth}
            \includegraphics[width=\linewidth]{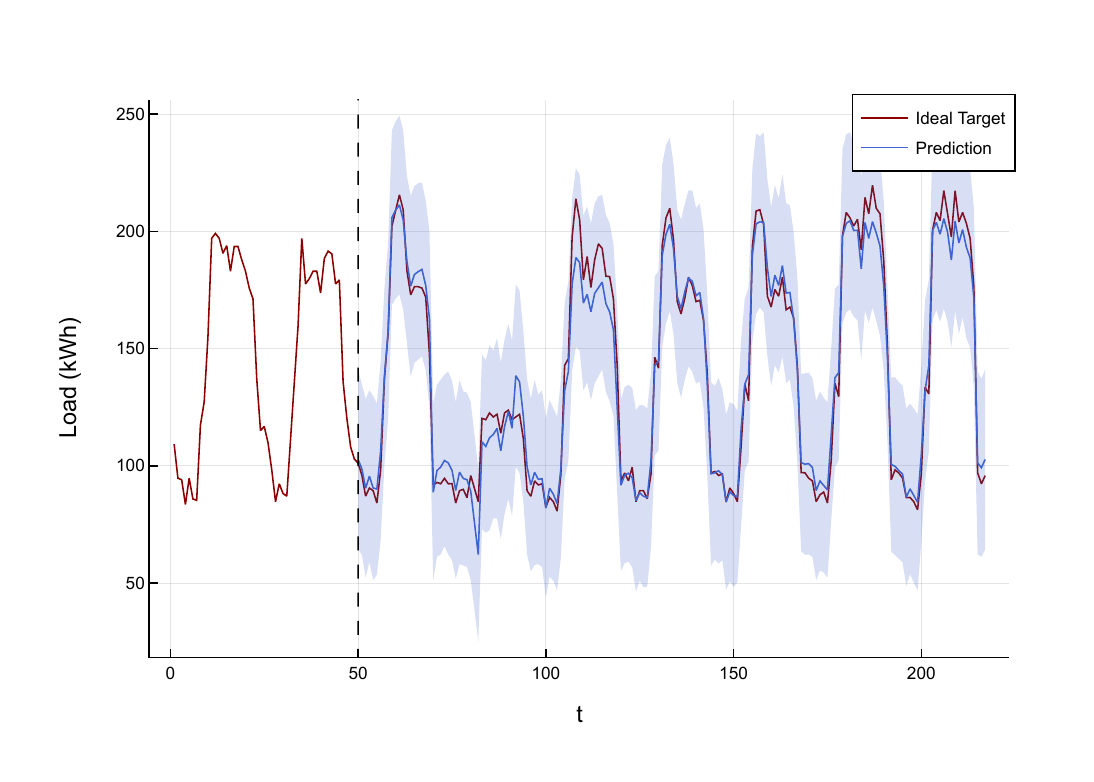}
            \caption{\small user 169}
        \end{subfigure}
        \par\bigskip\bigskip
        \begin{subfigure}{\linewidth}
            \includegraphics[width=\linewidth]{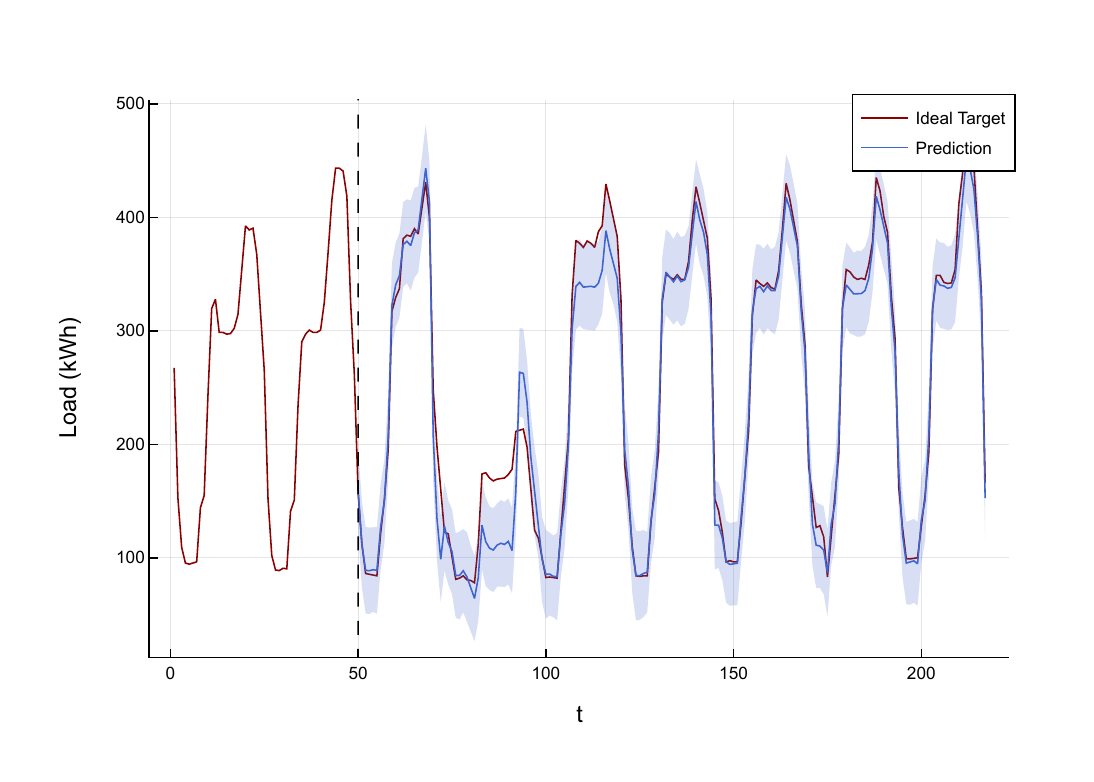}
            \caption{\small user 335}
        \end{subfigure}
    \end{minipage}
    \caption{7-day forecast for `electricity-c' consumption}\label{7-day forecast}
\end{figure}

\clearpage
\newpage
\subsubsection{Wind time series}

The wind \cite{sohier_wind_generation} dataset contains daily estimates of the wind energy potential for 28 countries over the period from 1986 to 2015, expressed as a percentage of a power plant's maximum output. The training was conducted using data from a 3-year period (1986-1989), and validation was performed on three months of data from 1992. Below, we present the results obtained for 30-day and 180-day predictions for different time series. For this purpose, we used a 3-layer RNN with 32 units in each layer and a 2-layer MLP with 64 units in each layer, both with ReLU activation functions. The final layer employs an identity activation function.
\begin{figure}[h]
   \begin{minipage}{0.45\textwidth}
        \centering
        \begin{subfigure}{\linewidth}
            \includegraphics[width=\linewidth]{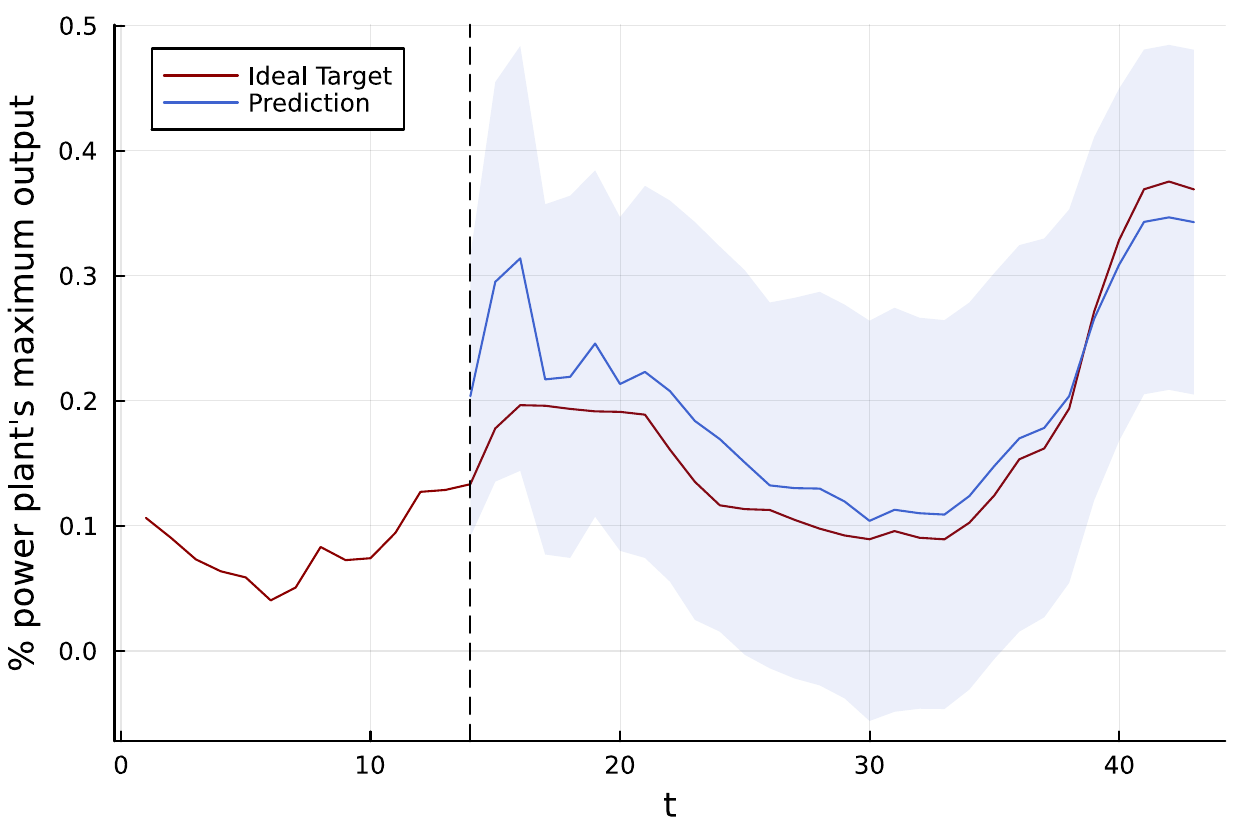}
            \caption{\small FR22}
        \end{subfigure}
        \par\bigskip\bigskip
        \begin{subfigure}{\linewidth}
            \includegraphics[width=\linewidth]{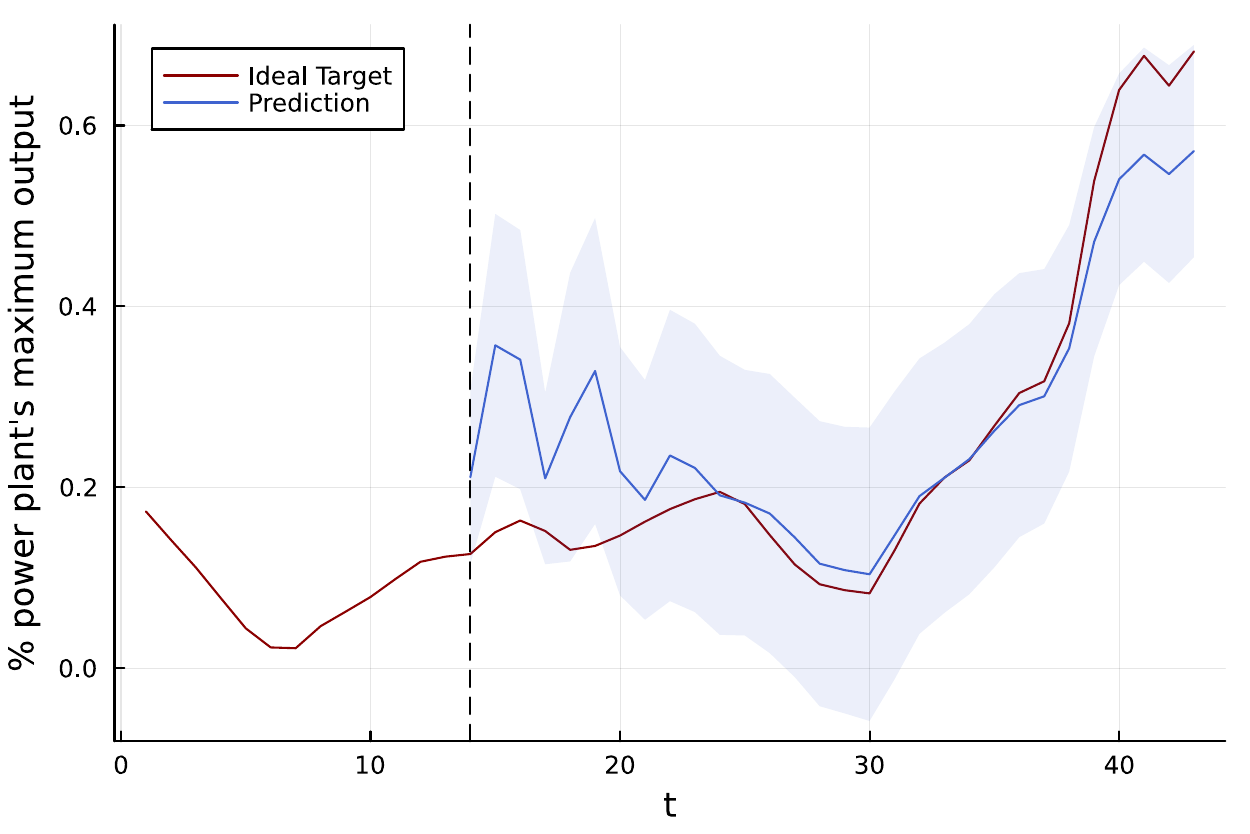}
            \caption{\small FR25}
        \end{subfigure}
    \end{minipage}
    \hfill
    \begin{minipage}{0.45\textwidth}
        \begin{subfigure}{\linewidth}
            \includegraphics[width=\linewidth]{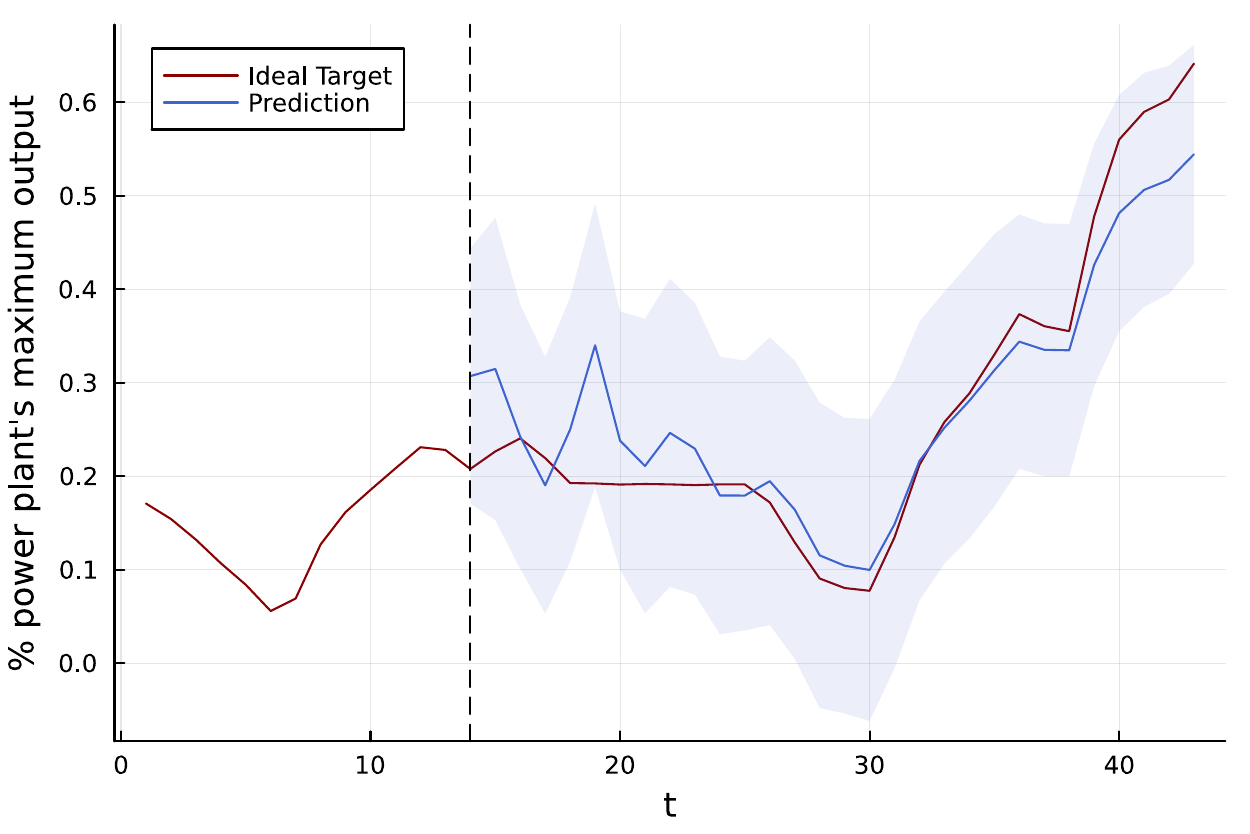}
            \caption{\small FR51}
        \end{subfigure}
        \par\bigskip\bigskip
        \begin{subfigure}{\linewidth}
            \includegraphics[width=\linewidth]{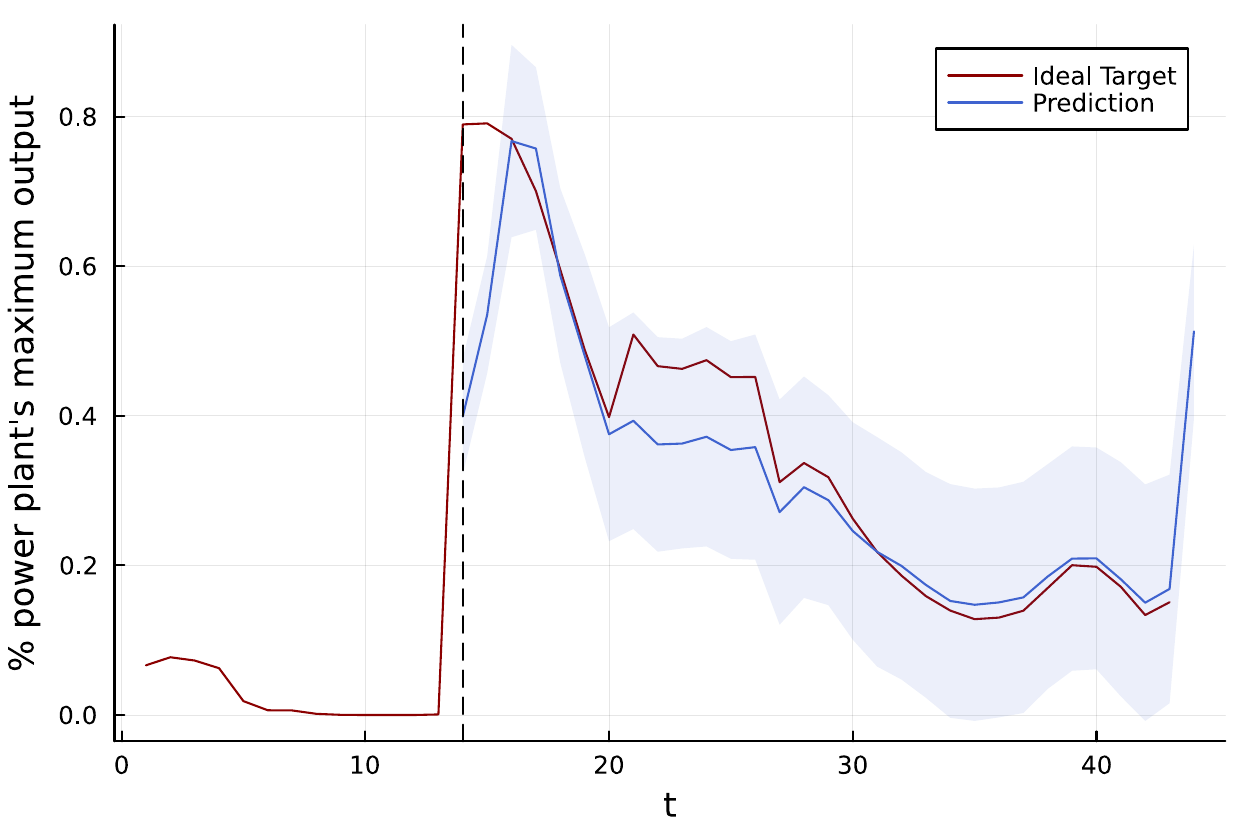}
            \caption{\small AT31}
        \end{subfigure}
    \end{minipage}
    \caption{30 days forecast prediction}
\end{figure}

\hfill

\begin{figure}[H]
\centering
   \begin{minipage}{0.6\textwidth}
        \begin{subfigure}{\linewidth}
            \includegraphics[width=\linewidth]{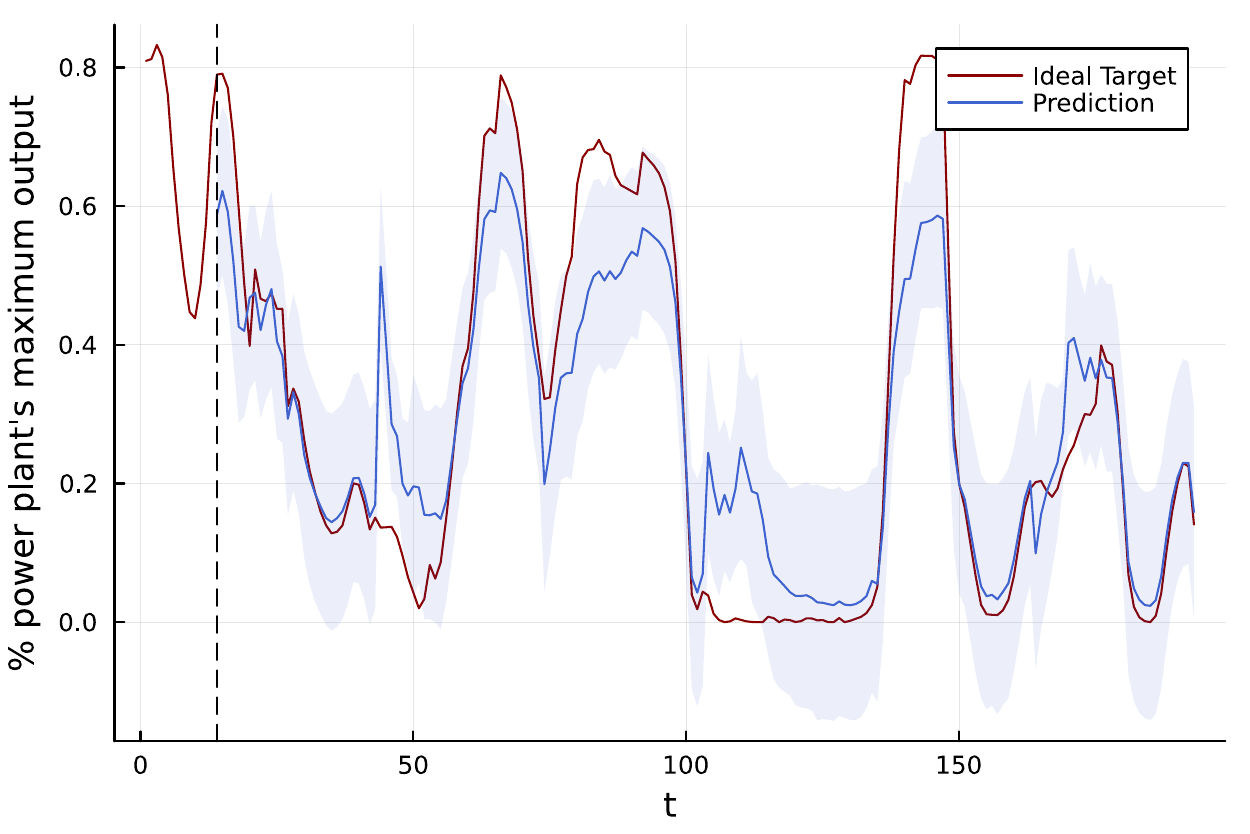}
            \caption{\small FR82}
        \end{subfigure}
    \end{minipage}
 \caption{$180$ days forecast prediction}
\end{figure}
\hfill
\begin{table}[h] \label{table Real World dataset 2}
\caption{\small Evaluation summary, using $\mathbf{QL_{\rho=0.5}/QL_{\rho=0.9}}$ metrics, on wind datasets, forecast window 30 day} \label{table Real World dataset 2}\label{Learning1D table}
\begin{center}
\setlength{\tabcolsep}{3pt}
\begin{tabular}{|c||c|}
\hline
\textbf{Method}  &\textbf{Wind} 
\\
 \toprule
 Prophet&$0.305/0.275$\\
 TRMF&$0.311/-$\\
 N-BEATS&$0.302/0.283$\\
 DeepAR&$0.286/0.116$\\
 ConvTras&$0.287/\bold{0.111}$\\
 ISL&$\bold{0.204}/0.245$\\
 \bottomrule
\end{tabular}\label{table Real World dataset 2}
\end{center}\label{table Real World dataset 2}
\end{table}

\section{Experimental Setup}\label{Experimental Setup}
All experiments were conducted using a personal computer with the following specifications: a MacBook Pro with macOS operating system version 13.2.1, an Apple M1 Pro CPU, and 16GB of RAM. The specific hyperparameters for each experiment are detailed in the respective sections where they are discussed. The repository with the code can be found in at the following link \url{https://github.com/josemanuel22/ISL}.

\end{document}